\documentclass[manuscript,screen]{acmart}
\pdfoutput=1
\AtBeginDocument{%
  \providecommand\BibTeX{{%
    \normalfont B\kern-0.5em{\scshape i\kern-0.25em b}\kern-0.8em\TeX}}}

\usepackage{amsmath,amsfonts,amsthm,graphicx,bm}
\usepackage{subfigure}
\usepackage[font=small,labelfont=bf,margin=0.5cm]{caption}
\usepackage{algpseudocode}

\usepackage[ruled,vlined]{algorithm2e}

\newcommand{\R}{\mathbb{R}}

\begin{document}

\title{Rare-Event Simulation for Neural Network and Random Forest Predictors}

\author{Yuanlu Bai}
\email{yb2436@columbia.edu}
\affiliation{%
  \institution{Columbia University}
  \streetaddress{500 W. 120th Street}
  \city{New York}
  \country{USA}
}

\author{Zhiyuan Huang}
\email{zhuang2@andrew.cmu.edu}
\affiliation{%
  \institution{Carnegie Mellon University}
  \country{USA}}

\author{Henry Lam}
\email{henry.lam@columbia.edu}
\affiliation{%
  \institution{Columbia University}
  \streetaddress{500 W. 120th Street}
  \city{New York}
  \country{USA}
}

\author{Ding Zhao}
\email{dingzhao@cmu.edu}
\affiliation{%
 \institution{Carnegie Mellon University}
  \country{USA}}

\renewcommand{\shortauthors}{Bai, Huang, Lam \& Zhao}

\begin{abstract}
 We study rare-event simulation for a class of problems where the target hitting sets of interest are defined via modern machine learning tools such as neural networks and random forests. This problem is motivated from fast emerging studies on the safety evaluation of intelligent systems, robustness quantification of learning models, and other potential applications to large-scale simulation in which machine learning tools can be used to approximate complex rare-event set boundaries. We investigate an importance sampling scheme that integrates the dominating point machinery in large deviations and sequential mixed integer programming to locate the underlying dominating points. Our approach works for a range of neural network architectures including fully connected layers, rectified linear units, normalization, pooling and convolutional layers, and random forests built from standard decision trees. We provide efficiency guarantees and numerical demonstration of our approach using a classification model in the UCI Machine Learning Repository.
\end{abstract}


\maketitle

\section{Introduction}

Due to the extensive development of artificial intelligence (AI), machine learning techniques have been embedded in many safety-sensitive physical systems, including autonomous vehicles \cite{koopman2017autonomous} and unmanned aircraft \cite{kochenderfer2012next}. In autonomous vehicles, for instance, machine learning predictors can be applied to many tasks including perception \cite{chen2015deepdriving,van2018autonomous}, path planning \cite{glasius1995neural,yang2004neural}, motion control \cite{spielberg2019neural}, or end-to-end driving systems \cite{muller2006off,chen2017end,kalra2016driving}. In these tasks, misprediction can cause catastrophic impacts on public safety, as exemplified by the series of fatal accidents encountered by autonomous driving systems due to the failures in detecting nearby vehicles or pedestrians (e.g. \cite{board2018preliminary,board2019report}). To reduce the risk of such catastrophe, machine learning models in these systems need to be carefully evaluated against safety, especially before their mass deployment in public. 

Recent research considers using probabilistic measures to quantify the risks of machine learning predictors or entire intelligent physical systems. These measures can be defined in a variety of ways. In \emph{robustness evaluation}, a prediction model, with neural network as a dominant example, is considered more robust if it is more likely to make a consistent prediction under small perturbations on the input \cite{goodfellow2014explaining}. When the perturbation is modeled via a random distribution, the robustness of neural networks is measured by the probability that the prediction value persists \cite{weng2018proven,webb2018statistical,wang2019statistically}. In more complex \emph{intelligent system evaluation}, risks can be quantified by the occurrence probabilities of safety-critical events. These events can be defined as the violation in terms of certain safety metrics (e.g., \cite{fraade2018measuring} listed seven potential safety metrics for autonomous vehicles including crashes per driving hour and disengagements per scenario), and recent studies use the probabilities of crash or injury in driving tasks as safety metrics \cite{zhao2016accelerated,huang2017accelerated,o2018scalable}. For AI-equipped autonomous vehicles, the evaluation target would implicitly involve a probabilistic measurement on the embedded machine learning model. Moreover, in \cite{wu2004learning}, neural networks are further used to approximate sophisticated safety-critical sets defined from complex system dynamics, and the target probabilities comprise hitting sets defined via these neural network outputs.

Our study is motivated from the estimation of probabilistic risk measures described above. Due to the complexity of machine learning predictors, these probabilities are typically unamenable to analytical formulas, even when the underlying stochastic distribution is fully modeled. This thus calls for the use of Monte Carlo simulation. However, the target probabilities, which signify the risks of dangerous yet unlikely events, are tiny. The problem thus falls into the domain of rare-event simulation, in which it is widely known that crude Monte Carlo can be extremely inefficient and variance reduction is necessarily employed. Traditionally, rare-event simulation techniques (e.g. \cite{bucklew2013introduction,juneja2006rare}) have been applied in broad application areas including queueing systems \cite{Sadowsky91,KN99,BGL10,blanchet2014rare,BMRT09,SP02,Ridder09,dupuis2009importance}, communication networks \cite{PARWAL89,KESWALCHA93,chen2019efficient},
finance \cite{Glasserman04,glasserman2008fast,glasserman2005importance}, insurance \cite{ASM00Ruin,pASM85a,collamore2002importance},
reliability \cite{Heidelberger95,RT09,Tuffin04QEST,nicola1993fast,nicola2001techniques},
biological processes \cite{Grassberger02,Sandmann09}, dynamical systems \cite{dupuis2012importance,vanden2012rare}, 
and combinatorics \cite{BKS07,Blanchet09AAP}. 
The evaluation of machine learning models and intelligent physical systems that we focus on here is a new application that is propelled rapidly by the growth of AI. Our goal is to provide a first step into building rare-event simulation algorithms in these applications, which integrate tools from both the disciplines of machine learning and rare-event simulation, and which are statistically guaranteed in terms of the classical efficiency notions in the rare-event literature.



More specifically, we study importance sampling (IS) \cite{siegmund1976importance} to design efficient estimators. In rare-event estimation, the rarity nature of hitting set dictates that crude Monte Carlo samples have a low frequency of observing the hitting occurrence, and this inefficiency exhibits statistically as a large relative error (i.e., ratio of standard deviation to mean) in the estimation. To mitigate this issue, IS uses an alternate distribution to generate samples that can attain a higher frequency in hitting the target event, and reweights the outputs to maintain unbiasedness via the likelihood ratios. To achieve a small relative error, the new generating distribution (i.e., the IS distribution) needs to be carefully selected, often by analyzing the weights in interaction with the hitting set geometry and the underlying system dynamics \cite{glynn1989importance,sadowsky1990large}. It is known that such analyses are important as ill-designed schemes can lead to significantly biased estimates \cite{glasserman1997counterexamples}. In this paper, we follow the above analysis path in the literature and use the common theoretical notion of efficiency called asymptotic optimality or logarithmic efficiency \cite{asmussen2007stochastic,juneja2006rare,Heidelberger95} that we will detail in the sequel. 



In terms of our scope of study, we focus on piecewise linear machine learning predictors, which include random forests and neural networks with common activation functions such as rectified linear units (ReLU). We also assume the underlying distribution is Gaussian or mixtures of such. Under this setting, we design provably efficient IS schemes to estimate rare-event probabilities that the prediction outputs hit above certain high thresholds. Our main methodology integrates the classical notion of dominating points \cite{sadowsky1990large,dieker2005asymptotically} for rare-event sets with sequential mixed integer programming (MIP) to attain an efficient estimator. Intuitively, a dominating point is the highest-density point in the rare-event set, so that using an IS distribution that shifts the mean to this point (via exponential tilting) gives rise to a distribution that hits the rare-event set more frequently and the generated likelihood ratio contributes properly to the probability of interest, which are desirable for controlling the relative error. However, this is only a local characterization. To explain, the simulation randomness stipulates that some generated samples may have huge likelihood ratios. Controlling these ratios in turn requires a geometric property that, in the Gaussian case, implies the dominating point to be on the boundary of the rare-event set, and that the latter lies completely inside one of the half-spaces cut by the tangential hyperplane passing through the dominating point (e.g., these occur when the rare-event set is convex). When this geometric property does not hold, then one needs to divide the rare-event set into union of smaller sets each bearing its own dominating point, and an efficient IS scheme is built via a mixture of exponential tiltings targeted at all these individual dominating points \cite{sadowsky1990large}. The sequential MIP in our procedure serves to locate all these dominating points. It casts each search as a density maximization problem constrained by hitting sets induced from the considered machine learning model. The involved feasible regions shrink sequentially as we add more ``cutting planes" to the constraints in order to remove the half-spaces that are already considered by earlier dominating points. Our MIPs are derived from the reformulation techniques that appeared recently in the machine learning literature, which leverage the geometric structures of ReLU neural networks \cite{tjeng2017verifying} and random forests \cite{mivsic2017optimization}. We provide a step-by-step guide in formulating random forests and different neural network architectures as suitable MIPs to be inserted into our sequential algorithm. We also provide theoretical results on asymptotic optimality that targets at general piecewise polyhedrals where applies to our considered rare-event sets. Towards this, we also derive large deviations results for the associated probabilities of interest.

The paper is organized as follows. Section \ref{sec:lit} first provides a brief literature review. Section \ref{sec:problem_setting} describes our problem setting and notations. Section \ref{sec:cuttingplain} presents our algorithm and theoretical guarantees. Section \ref{sec:fomulation} provides the MIP formulations for random forests and different neural network architectures. Section \ref{sec:experiments} shows numerical results. Section \ref{sec:guarantee} contains the proofs of theorems.




\section{Related Work}\label{sec:lit}
A significant line of work studies the use of large deviations to invent efficient IS procedures, which mathematically identifies the most likely path to trigger a rare event through minimizing the so-called rate function (see, e.g., the surveys \cite{bucklew2013introduction,asmussen2007stochastic,rubinstein2016simulation,glasserman2013monte,juneja2006rare,blanchet2012state}). Among these studies, our approach extends the IS schemes using dominating points \cite{sadowsky1990large,dieker2005asymptotically}. Similar idea of using half-spaces to split rare-event set is also considered in \cite{ahn2018efficient,owen2019importance} where the rare-event set is constrained to be a union of half-spaces and the half-spaces are explicitly given in the setting. The use of sequential MIP algorithm on an implicitly half-space separable rare-event set distinguishes our work from these studies. To prove the efficiency of our algorithm, we need to derive the asymptotic result for the rare-event probability of interest. \cite{Hashorva2003gaussian} provides a variety of useful techniques to represent the asymptotic approximation of probability using dominating points, but some technical modifications need to be made to fit in our settings. Similar to our derivations, \cite{honnappa2018dominating} represents the asymptotic of probability on convex sets using dominating points, yet focuses on a different scaling setting from ours.

Other IS schemes include the cross-entropy method  \cite{RUBKRO04,DNR00,DEBKROMANRUB05,Rubinstein97,Rubinstein99} that uses sequential stochastic optimization to search for an optimal IS distribution in a parametric family. Adaptive IS \cite{KBCP99,DG01,ABJ06,BJK04} updates the IS distribution iteratively between simulated replications to approach the optimal (zero-variance) IS distribution and generates non i.i.d samples for estimating the target expectation associated with finite-state discrete Markov chains. Another line of studies use techniques such as Markov-chain Monte Carlo (MCMC) to sample from the rare-event set of interest, or approximately from the conditional distribution given the occurrence of the rare event \cite{botev2013markov,grace2014automated,chan2012improved,botev2020sampling}. IS schemes have also been designed for heavy-tailed systems \cite{blanchet2008efficient,blanchet2008state,dupuis2007importance,chen2019efficient,blanchet2012efficient,murthy2015state,hult2012importance}, in contrast to the light-tailed settings considered in this paper. Besides IS, other competing methods for rare-event simulation include conditional Monte Carlo \cite{AsmBin97,AsmKro06} and splitting \cite{LLLT09,GHSZ99,GLAHEISHAZAJ98,Dean2009562,glasserman1998large}.

In machine learning literature, some studies discuss using probability measure to evaluate the robustness of prediction models. Since the measure can be extremely small, rare-event simulation techniques are considered in these studies. \cite{webb2018statistical} discusses an adaptive multilevel splitting approach to estimate the statistical robustness of machine learning models. \cite{uesato2018rigorous} proposes to learn a failure probability predictor to approximate the minimum variance IS distribution in estimating agent failure probabilities. \cite{weng2018proven} proposes an approach to compute the lower and upper bounds for a probabilistic robustness measure. The topic of these works is one of our key motivations, and our work can be viewed as a step towards the provision of rigorous guarantees for methodologies driven by these applications.

Lastly, another related line of research studies optimization problems with machine learning models in the objective. \cite{mivsic2017optimization} discusses the optimization of tree ensemble models and provides treatment for large scale problems. \cite{tjeng2017verifying} formulates the robustness verification of neural networks as MIP problems. These studies leverage the piecewise linear property of these machine learning models to turn optimization on the prediction output into tractable MIPs. Our MIP formulations for finding dominating points follow from these optimization studies.


\section{PROBLEM SETTING}\label{sec:problem_setting}
\subsection{Rare-Event Probability Estimation}\label{subsec:estimation}

We state our problem setting. Consider a prediction model $g(\cdot)$, with input $X \in \mathbb R^d$ and output $g(X) \in \mathbb R$. Suppose that the input follows a Gaussian distribution, i.e, $X \sim N(\mu,\Sigma)$, where $\Sigma$ is a $d\times d$ positive definite matrix. We want to estimate the probability $p=P(g(X) \geq  \gamma)$, where $\gamma \in \mathbb R$ is a threshold that triggers a rare event. We note that the Gaussian assumption can be relaxed without much difficulty in our framework to, for instance, mixtures of Gaussians, which we will discuss later and can expand our scope of applicability. 

This problem setting is related to risk assessments involving machine learning models, as exemplified below.

\begin{example}[Statistical Robustness Metric \cite{webb2018statistical,weng2018proven}] \label{example:robust_metric}
Consider a classification model that predicts using ``score functions'' $g_i(\cdot)$ with $i=1,..,K$ where $K$ denotes the number of categories. The predicted output is the category that has the maximum score, i.e. the prediction at $X$ is given by $\arg \max_i g_i(X)$. Suppose an example input $x_0$ belongs to category $c$. A classification model is robust if it gives correct prediction for all $x$ such that $d(x,x_0)\leq \epsilon$ where $d$ denotes a certain distance and $\epsilon>0$ is a small real number. A statistical robustness metric considers $p=P(\max_i g_i(X)-g_c(X)\geq 0)$, where $X$ follows a distribution concentrated around $x_0$. Here $p$ represents the probability that the output is inconsistent with the baseline prediction at $x_0$. 
\end{example}

\begin{example}[Risk Evaluation of Intelligent Physical Systems \cite{dreossi2019compositional}]
\label{example:b}
Consider an intelligent physical system that embeds a machine learning predictor $g$, so that the decision of the system given an input $X$ can be expressed as $h(g(X))$.  The probability $P(h(g(X))\in S)$, where $S$ represents a risky region, can be used to measure the risk of the system decision. In most cases, $h$ is random by itself and can have a different complexity structure than the function class $g$. In this paper, we consider a rare-event probability in the form of $P(g(X)\geq\gamma)$ as a first step of study in this direction.  
\end{example}

\begin{example}[Probability Evaluation for Learned Rare-Event Set \cite{wu2004learning}] \label{example:a}
When the rare-event set is very complicated (e.g., in autonomous driving contexts), one approach to retain tractability is to approximate or learn the rare-event set via classification tools. Given historical or simulated data $\{X,Y \}$, where $Y\in \{0,1 \}$ denotes whether a rare-event (e.g. a crash) occurs to the system of interest under input $X$, we train a neural network $g(\cdot)$ to classify the rare-event region given $X$. The learned rare-event set is represented by $\{x:g(x)\geq \gamma\}$, where $\gamma$ is the threshold for classifying rare-event (e.g. $\gamma=0.5$). Since $\{x:g(x)\geq \gamma\}$ is an approximation of the true rare-event set, $p=P(g(X)\geq \gamma)$ provides an approximation on the probability of the rare event. 
\end{example}

It is known that neural networks can be vulnerable to adversarial attacks, in that a tiny perturbation in the input can exert a large effect on the prediction output \cite{goodfellow2014explaining}, and such a perturbed input is considered as an adversarial example. Studies have discussed how to find these adversarial examples \cite{kurakin2016adversarial1} and to conduct adversarial learning \cite{kurakin2016adversarial}. Among them, Example \ref{example:robust_metric} is an example of a probabilistic measure on how likely adversarial examples appear around a certain input. Examples \ref{example:b} and \ref{example:a}, on the other hand, represent endeavors to tackle safety-critical problems driven by applications involving AI systems, which can embed machine learning models or are approximated by them.



\subsection{Importance Sampling}\label{subsec:IS}

When $p$ is small, estimation using crude Monte Carlo is challenging since, intuitively, the samples have a low frequency of hitting the target set. This is statistically manifested as a large relative error. To be more specific, suppose that we use the crude Monte Carlo estimator $\hat{p}_N=\frac{1}{N}\sum_{i=1}^N I(g(X_i)\geq\gamma)$ to estimate $p$. Since the probability $p$ is tiny, the error of the estimator should be measured relative to the size of $p$. In other words, we would like the probability of having a large relative error to be small, i.e., $P(|\hat{p}_N-p|>\varepsilon p)\leq \delta$ where $\delta$ is the confidence level and $0<\varepsilon<1$. By Markov's inequality, a sufficient condition for this is 
$$N\geq \frac{Var(I(g(X)\geq\gamma))}{\delta\varepsilon^2 E[I(g(X)\geq \gamma)]^2}=\frac{RE^2}{\delta\epsilon^2}.$$ 
where $RE=\sqrt{Var(I(g(X)\geq\gamma))}/E[I(g(X)\geq\gamma)]$ is the relative error. For the crude Monte Carlo estimator, the RE is given by $\sqrt{(1-p)/p}$. 
That is, the simulation size $N$ has to be roughly proportional to $1/p$ in order to achieve a given relative error. Under the settings that $X$ has a Gaussian distribution and $g$ is piecewise linear (see Corollary \ref{thm:asymptotic}), $p$ is exponentially small in the threshold level $\gamma$, and hence the required simulation size would grow exponentially in $\gamma$.
\par 
A common approach to speed up simulation in such contexts is to use IS (see, e.g. the surveys \cite{bucklew2013introduction,asmussen2007stochastic,rubinstein2016simulation,glasserman2013monte,juneja2006rare,blanchet2012state}, among others). Suppose $X$ has a density $f$. The basic idea of IS is to change the sampling distribution to say $\tilde f$, and output
\begin{equation} \label{eq:is_estimator}
Z=I(g(\tilde{X}) \geq  \gamma) \frac{f(\tilde{X})}{\tilde f(\tilde{X})}, 
\end{equation} 
where $\tilde{X}$ is sampled from $\tilde f$. This output is unbiased if $f$ is absolutely continuous with respect to $\tilde f$ over the rare-event set $\{x:g(x)\geq\gamma\}$. By choosing $\tilde f$ appropriately, one can substantially reduce the simulation variance.

To measure the efficiency of an IS scheme, we introduce a rarity parameter, say $\gamma$, that parametrizes the rare-event probability $p_\gamma$ such that $p_\gamma\to0$ as $\gamma\rightarrow \infty$. As discussed before, since the probability of interest is small, one should focus on the relative error of the Monte Carlo estimator with respect to the magnitude of this probability. To this end, we call an IS estimator $Z_\gamma$ for $p_\gamma$ asymptotically optimal \cite{asmussen2007stochastic,juneja2006rare} if \begin{equation}
 \lim_{\gamma\rightarrow \infty} \frac {\log \tilde{E}[Z_\gamma^2]}{\log \tilde{E}[Z_\gamma]} = 2,\label{asymptotic efficiency}
\end{equation}
where $\tilde{E}$ denotes the expectation with regard to $\tilde f$. The notion \eqref{asymptotic efficiency} is equivalent to saying that $\tilde{E}[Z_\gamma^2]/\tilde{E}[Z_\gamma]^2$ is at most polynomially growing in $\gamma$. This ensures that the second moment, or the variance, does not explode exponentially relative to the probability of interest as $\gamma$ increases, thus preventing an exponentially large number of simulation replications to achieve a given relative accuracy. We will use asymptotic optimality as our efficiency criterion in this paper.
\par 
Another commonly used efficiency criterion is the bounded relative error, which is defined as 
$$
\limsup_{\gamma\rightarrow\infty}\frac{\tilde{E}[Z_\gamma^2]}{\tilde{E}[Z_\gamma]^2}<\infty.
$$
This is a stronger condition than asymptotic optimality. More efficiency criteria can be found in \cite{juneja2006criteria, ecuyer2010criteria}.



\section{Efficient Importance Sampling via Sequential Mixed Integer Programming}
\label{sec:cuttingplain}

In the case of Gaussian input distributions, finding a good $\tilde f$ is particularly handy and one approach to devise good IS distributions uses the notion of so-called dominating point. As explained in the introduction, a dominating point can be understood as the highest-density point in the rare-event set that satisfies some conditions. More precisely, the collection of dominating points for a rare-event set with Gaussian distributed input is defined in Definition \ref{def:dominant}.
\begin{definition}
Suppose that a set $A\subset\R^d$ satisfies that 
$S\subset \bigcup_{a\in A}\{x:(a-\mu)^T\Sigma^{-1}(x-a)\geq 0\}$ and that $a=\arg\min_x\{(x-\mu)^T\Sigma^{-1}(x-\mu):x\in S \text{ and } (a-\mu)^T\Sigma^{-1}(x-a)\geq 0\}$ for any $a\in A$. Moreover, suppose that the above conditions do not hold anymore if we remove any element from $A$. Then the points in $A$ are called the dominating points of $S$ with input distribution $N(\mu,\Sigma)$.
\label{def:dominant}
\end{definition}

\par 
Note that minimizing $(x-\mu)^T \Sigma^{-1}(x-\mu)$ is equivalent to maximizing $\phi(x;\mu,\Sigma)$, the Gaussian density with mean $\mu$ and covariance $\Sigma$. The condition $2(a-\mu)^T\Sigma^{-1}(x-a) \geq 0$ is the first-order condition of optimality for the optimization $\min_x (x-\mu)^T \Sigma^{-1}(x-\mu) $ over a convex set for $x$. Thus, intuitively, each dominating point in the collection $A$ can be viewed as the highest-density point in a ``local" region formed by $S\cap \{x: (a-\mu)^T\Sigma^{-1}(x-a)\geq 0\}$. In particular, if $\{x:g(x)\geq\gamma\}$ is a convex set, then there is only one dominating point $a$. In this case, a well-known IS scheme is to use a Gaussian distribution $N(a,\Sigma)$ as the IS distribution $\tilde{f}$. 
\par 
We explain intuitively why we need more than one dominating point (the highest-density point over $S$) and the pitfall if we omit the other ones in constructing efficient IS. Suppose that the rare-event set consists of two disconnected convex components which are nearly equi-distant with respect to the origin, and we choose the IS distribution to be centered at the dominating point of one component. Then, if a sample from the IS distribution hits the other component, a scenario that could be unlikely but possible, the resulting likelihood ratio, which now contributes to the output as the rare-event set is hit, could possibly be tremendous. This ultimately leads to an explosion of the relative error in the IS estimator. \cite{glasserman1997counterexamples} presents more counterexamples which show that it is essential to find all the dominating points in constructing an efficient IS based on mixtures.
\par 
In view of the aforementioned discussions, we consider the following IS scheme. If we can split $\{x:g(x) \geq \gamma\}$ into $\mathcal{R}_1,...,\mathcal{R}_r$, and for each $\mathcal{R}_i,i=1,...,r$ there exists a dominating point $a_i$ such that $a_i=\arg \min_{x} \{(x-\mu)^T \Sigma^{-1}(x-\mu) : x \in \mathcal R_i \} $ and $\mathcal R_i  \subseteq \{x:(a_i-\mu)^T\Sigma^{-1}(x-a_i) \geq 0\}$, then we use a Gaussian mixture distribution with $r$ components as the IS distribution $\tilde{f}$, where the $i$th component has mean $a_i$. This proposal guarantees the asymptotic optimality of the IS (see Theorem \ref{thm:efficiency}). 
\par 


In our task, because the machine learning predictor $g(x)$ is nonlinear and $x$ is high-dimensional in general, splitting $\{x:g(x) \geq \gamma\}$ into $\mathcal{R}_1,...,\mathcal{R}_r$ that have dominating points is challenging even with known parameters. This challenge motivates us to use Algorithm \ref{algo:main} to obtain the dominating points $a_1,...,a_r$ that constructs an efficient IS distribution. The procedure uses a sequential ``cutting plane" approach to exhaustively look for all dominating points, by reducing the search space at each iteration via taking away the regions covered by found dominating points. The set $A$ in the procedure serves to store the dominating points we have located throughout the procedure. At the end of the procedure, we obtain a set $A$ that contains all the dominating points $a_1,...,a_r$. 

\begin{algorithm}[h]
\KwIn{Prediction model $g(x)$, threshold $\gamma$, input distribution $N(\mu,\Sigma)$.}
\KwOut{dominating point set $A$.}
\nl Start with $A = \emptyset$;\\

\nl {\bf While } $\{x: g(x) \geq \gamma , (a_i-\mu)^T\Sigma^{-1}(x-a_i) <0, \mbox{ $\forall a_i \in A$} \} \neq  \emptyset$ {\bf do } \label{algo:while_line}

\nl \ \ \ \ \ \ Find a dominating point $a$ by solving the optimization problem 
\begin{align} 
\label{eq:opt_ite}
		a=\arg \min_{x} &\ \   (x-\mu)^T \Sigma^{-1}(x-\mu) &\\ \notag
s.t. &\ \  g(x) \geq \gamma &\\  \notag
&\ \  (a_i-\mu)^T\Sigma^{-1}(x-a_i) <0,\ \mbox{$\forall a_i \in A$}&  \notag
\end{align} 
\ \ \ \ and update $A \leftarrow A \cup \{a\}$;

\nl \bf End\
    \caption{{\bf Procedure to find all dominating points for the set $\{x: g(x)\geq \gamma \}$.} \label{algo:main}}

\end{algorithm}

Algorithm \ref{algo:main} gives $A=\{a_1,\dots,a_r\}$. With this, we split $\{x:g(x)\geq \gamma\}$ into $\mathcal{R}_1,\dots,\mathcal{R}_r$ where $\mathcal{R}_i=\{x:g(x)\geq\gamma,(a_i-\mu)^T\Sigma^{-1}(x-a_i)\geq 0,(a_j-\mu)^T\Sigma^{-1}(x-a_j)\leq 0,\forall j<i\}$. Clearly $a_i=\arg\min\{(x-\mu)^T\Sigma^{-1}(x-\mu):x\in\mathcal{R}_i\}$ and $(a_1-\mu)^T\Sigma^{-1}(a_1-\mu)\leq \dots\leq (a_r-\mu)^T\Sigma^{-1}(a_r-\mu)$. Moreover, we note that $(a_1-\mu)^T\Sigma^{-1}(a_1-\mu)=\min_{i=1,\dots,r}\{(a_i-\mu)^T\Sigma^{-1}(a_i-\mu)\}$.
\par 
Given the dominating point set $A$ we use 
$$
\frac{1}{r}N(a_1,\Sigma)+\dots+\frac{1}{r}N(a_r,\Sigma)
$$
as the IS distribution. That is, the IS estimator is 
\begin{equation}
    \label{eq:proposed_IS}
    Z=I(g(\tilde{X})\geq\gamma)L(\tilde{X})
\end{equation}
where $\tilde{X}\sim \tilde{f}$ and $L$, the likelihood ratio, is defined as 
$$
L(x)=\frac{f(x)}{\tilde{f}(x)}=\frac{re^{-\frac12 (x-\mu)^T\Sigma^{-1}(x-\mu)}}{e^{-\frac12(x-a_1)^T\Sigma^{-1}(x-a_1)}+\cdots+e^{-\frac12(x-a_r)^T\Sigma^{-1}(x-a_r)}}.
$$
\par 
As a summary, after computing the dominating points $A=\{a_1,\dots,a_r\}$ using Algorithm \ref{algo:main}, we estimate the probability of interest via Algorithm \ref{algo:IS}.
\begin{algorithm}[h]
\KwIn{Prediction model $g(x)$, threshold $\gamma$, dominating points $A=\{a_1,\dots,a_r\}$, simulation size $N$.}
\KwOut{Estimated rare-event probability $\hat{p}$.}
\nl Generate $\tilde{X}_1,\dots,\tilde{X}_N\sim \frac{1}{r}N(a_1,\Sigma)+\dots+\frac{1}{r}N(a_r,\Sigma)$;\\
\nl Compute $\hat{p}=\frac{1}{N}\sum_{i=1}^N I(g(\tilde{X}_i)\geq\gamma)L(\tilde{X}_i)$ where 
$$
L(x)=\frac{re^{-\frac12 (x-\mu)^T\Sigma^{-1}(x-\mu)}}{e^{-\frac12(x-a_1)^T\Sigma^{-1}(x-a_1)}+\cdots+e^{-\frac12(x-a_r)^T\Sigma^{-1}(x-a_r)}};
$$\\
\nl \bf End\
    \caption{{\bf Construct the IS estimator with all the dominating points.} \label{algo:IS}}

\end{algorithm}
\par
The efficiency guarantee of the proposed IS estimator \eqref{eq:proposed_IS} is given by:
\begin{theorem}
Suppose that the input $X\sim N(\mu,\Sigma)$ and the prediction model $g(\cdot)$ is a piecewise linear function such that $P(g(X)\geq\gamma)>0$ for any $\gamma\in\R$. The IS estimator $Z$ is defined in \eqref{eq:proposed_IS}. Then we have that $\tilde{E}[Z^2]/\tilde{E}[Z]^2$ is at most polynomially growing in $\gamma$. That is, $Z$ is asymptotically optimal.
\label{thm:efficiency}
\end{theorem}

Theorem \ref{thm:efficiency} is proved by constructing an upper bound for the relative error, which in turn depends on the asymptotic approximation of probability on polytope sets using dominating points. Our proof leverages the results in \cite{Hashorva2003gaussian} on the tail exceedance asymptotic of $P(N(0,\Sigma_n)\geq t_n)$ where $\Vert t_n\Vert\rightarrow\infty$ as $n\rightarrow\infty$, but requires substantial generalization. Note that Theorem \ref{thm:efficiency} only makes the very general assumptions that $g$ is piecewise linear and the probability $P(g(X)\geq\gamma)$ is nondegenerate (i.e., non-zero) for any $\gamma\in\R$. Our result applies to, for example, the probability $P(AX\geq t)$ where $A$ is a constant matrix and $t-\gamma e_1$ is a constant vector (here, $e_1=(1,0,\dots,0)^T$). If $AA^T$ is not invertible, then it is not easily reducible to the setting studied in \cite{Hashorva2003gaussian}. To achieve a general result, we carefully construct a superset and a subset of the rare-event set to derive tight enough upper and lower bounds for the probability of interest, in which we analyze the involved asymptotic integrals instead of using the conditional probability representation in \cite{Hashorva2003gaussian} that is not directly applicable in our setting.
For the detailed proof, please refer to Section \ref{sec:guarantee}. 

A by-product in deriving Theorem \ref{thm:efficiency} is the large deviations probability asymptotic for $P(g(X)\geq\gamma)$:
\begin{corollary}
Suppose that the input $X\sim N(\mu,\Sigma)$ and the prediction model $g(\cdot)$ is a piecewise linear function such that $P(g(X)\geq\gamma)>0$ for any $\gamma\in\R$. Denote $a=\arg\min\{(x-\mu)^T\Sigma^{-1}(x-\mu):g(x)\geq\gamma\}$. Then $-\log P(g(X)\geq\gamma)=(1+o(1)) (a-\mu)^T\Sigma^{-1}(a-\mu)/2$ as $\gamma\rightarrow\infty$. In particular, $P(g(X)\geq\gamma)$ is exponentially small in $\gamma$.
\label{thm:asymptotic}
\end{corollary}

The theoretical guarantee given by Theorem \ref{thm:efficiency} justifies the sequential MIP algorithm for searching dominating points. The resulting mixture IS distribution is asymptotically optimal. We point out some related works that use mixture distributions that are related to our proposed method. In \cite{ahn2018efficient,owen2019importance}, mixture IS distributions are constructed based on separating rare-event set with half-spaces. However, in these works, the rare-event set is restricted to be a union of half-spaces, and these half-spaces are assumed to be known. The use of Algorithm \ref{algo:main} allows us to deal with more general rare-event sets. Moreover, in relation to Corollary \ref{thm:asymptotic}, we also mention the work \cite{honnappa2018dominating} that derives an asymptotic result for Gaussian probabilities using dominating points. However, they focus on convex hitting sets where the entire set is scaled with a rarity parameter, which is different from our settings. First, our rare-event set is not necessarily convex. Second, even if we separate our rare-event set into the union of convex sets, their results still cannot be applied, since in our settings some linear constraints are allowed to be fixed instead of scaling with $\gamma$.


The proposed IS scheme can be extended to problems with Gaussian mixture inputs. Suppose the Gaussian mixture has $m$ components, so that $X \sim \sum_{j=1}^m \pi_j N(\mu_j,\Sigma_j)$. For each component $j$, we implement Algorithm \ref{algo:main} with input distribution $N(\mu_j,\Sigma_j)$ to obtain dominating point set $A_j$ (with cardinality $r_j$). The proposed IS distribution is given by $\tilde f(x)=\sum_{j=1}^m \sum_{i=1}^{r_j}  1/r_j \pi_j N(a_{ji},\Sigma_j) $. We summarize the procedure as Algorithm \ref{algo:mixture}.
\begin{algorithm}[h]
\KwIn{Prediction model $g(x)$, threshold $\gamma$, input distribution $\sum_{j=1}^m \pi_j N(\mu_j,\Sigma_j)$, simulation size $N$.}
\KwOut{Estimated rare-event probability $\hat{p}$.}
\nl Implement Algorithm \ref{algo:main} with input distribution $N(\mu_j,\Sigma_j)$ to get $A_j=\{a_{j1},\dots,a_{jr_j}\}$;\\
\nl Generate $\tilde{X}_1,\dots,\tilde{X}_N\sim \sum_{j=1}^m\sum_{i=1}^{r_j}1/r_j\pi_j N(a_{ji},\Sigma_j)$;\\
\nl Compute $\hat{p}=\frac{1}{N}\sum_{i=1}^N I(g(\tilde{X}_i)\geq\gamma)L(\tilde{X}_i)$ where 
\begin{equation}
L(x)=\frac{\sum_{j=1}^m \pi_j|\Sigma_j|^{-\frac{1}{2}} e^{-\frac{1}{2}(x-\mu_j)^T\Sigma_j^{-1}(x-\mu_j)}}{\sum_{j=1}^m\sum_{i=1}^{r_j}1/r_j\pi_j|\Sigma_j|^{-\frac{1}{2}} e^{-\frac{1}{2}(x-a_{ji})^T\Sigma_j^{-1}(x-a_{ji})}};
\label{eqn:mixlikelihood}
\end{equation}\\
\nl \bf End\
    \caption{{\bf Procedure for Gaussian mixture distributed input.} \label{algo:mixture}}

\end{algorithm}

Similar to Algorithm \ref{algo:IS}, we have the efficiency guarantee for Algorithm \ref{algo:mixture}:
\begin{corollary}
Suppose that the input $X \sim \sum_{j=1}^m \pi_j N(\mu_j,\Sigma_j)$ and the prediction model $g(\cdot)$ is a piecewise linear function such that $P(g(X)\geq\gamma)>0$ for any $\gamma\in\R$. The IS estimator $Z$ is defined as $I(g(\tilde{X})\geq\gamma)L(\tilde{X})$ where $\tilde{X}\sim \sum_{j=1}^m\sum_{i=1}^{r_j}1/r_j\pi_j N(a_{ji},\Sigma_j)$ and $L(x)$ is as defined in \eqref{eqn:mixlikelihood}. Then we have that $\tilde{E}[Z^2]/\tilde{E}[Z]^2$ is at most polynomially growing in $\gamma$. That is, $Z$ is asymptotically optimal.
\label{thm:corollary}
\end{corollary}

When we apply Algorithm \ref{algo:main} to find all dominating points, the key is to be able to solve the optimization problems in \eqref{eq:opt_ite}. We will investigate this in the next section.

\section{TRACTABLE OPTIMIZATION FORMULATION FOR PREDICTION MODELS} 
\label{sec:fomulation}
We discuss how to formulate the optimization problems in Algorithm \ref{algo:main} as an MIP with quadratic objective function and linear constraints. Sections \ref{sec:RF} and \ref{sec:neural} focus on random forest and neural network structures respectively. 

\subsection{Tractable Formulation for Random Forest}\label{sec:RF}

To look for dominating points in a random forest or tree ensemble,  we follow the route in \citeN{mivsic2017optimization} that studies optimization over these models. We consider a random forest as follows. The input $x$ has $d$ dimensions. Suppose the model consists of $T$ trees $f_1,...,f_T$. In each tree  $f_t$, we use $a_{i,j}$ to denote the $j$th unique split point for the $i$th dimension of the input $x$, such that $a_{i,1}<a_{i,2}<...<a_{i,K_i}$, where $K_i$ is the number of unique split points for the $i$th dimension of $x$. 

Following the notations in \citeN{mivsic2017optimization}, let $\textbf{leaves}(t)$ be the set of leaves (terminal nodes) of tree $t$ and $\textbf{splits}(t)$ be the set of splits (non-terminal nodes) of tree $t$. In each split $s$, we let $\textbf{left}(s)$ be the set of leaves that are accessible from the left branch (the query at $s$ is true), and $\textbf{right}(s)$ be the set of leaves that are accessible from the right branch (the query at $s$ is false). For each node $s$, we use $\textbf{V}(s)\in \{1,...,d\}$ to denote the dimension that participate in the node and $\textbf{C}(s)\in \{1,...,K_{\textbf{V}(s)}\}$ to denote the set of values of dimension $i$ that participate in the split query of $s$ ($\textbf{C}(s)=\{j\}$ and $\textbf{V}(s)=\{i\}$ indicate the query $x_{i} \leq a_{i,j} $). We use $\lambda_t$ to denote the weight of tree $t$ ($\sum_{t=1}^T \lambda_t=1$). For each  $l \in \textbf{leaves}(t)$, $p_{t,l}$ denotes the output for the $l$th leaf in tree $t$.

To formulate the random forest optimization as an MIP, we introduce binary decision variables $z_{i,j}$ and $y_{t,l}$. First, we have \begin{equation}\label{eq:x_to_z}
z_{i,j}=I(x_i \leq a_{i,j}),\ i=1,...,d,\ j=1,...,K_i.
\end{equation}
We then use $y_{t,l}=1$ to denote that tree $t$ outputs the prediction value $p_{t,l}$ on leaf $l$, and $y_{t,l}=0$ otherwise. We use $\textbf{z},\textbf{y}$ to represent the vectors of $z_{i,j}$ and $y_{t,l}$ respectively. For the input $x$, we assume that $x \in [-B,B]^d$ and $|a_{i,j}| \leq B$. Then \eqref{eq:x_to_z} is represented by the following constraints \begin{align*}
 & x_i \leq a_{i,j} +2(1-z_{i,j}) B \\
 & x_i > a_{i,j} -2z_{i,j} B.
\end{align*}

Now we formulate \eqref{eq:opt_ite} with $A=\emptyset$ as the following MIP \begin{align} \label{eq:forest}
\min_{x,\textbf{y},\textbf{z}} &\ \ (x-\mu)^T\Sigma^{-1}(x-\mu)  &\\ \notag
s.t. &\ \     \sum_{t=1}^T \sum_{l\in \textbf{leaves}(t)} \lambda_t p_{t,l} y_{t,l}   \geq \gamma   & \\ \notag
&\ \  \sum_{l\in \textbf{leaves}(t)} y_{t,l}=1 ,\ \forall t \in \{1,...,T\} &\\ \notag
&\ \  \sum_{l\in \textbf{left}(s)} y_{t,l} \leq \sum_{j\in \textbf{C}(s)} z_{\textbf{V}(s),j},\ \forall t \in \{1,...,T\},\ s\in  \textbf{splits}(t)&\\ \notag
&\ \  \sum_{l\in \textbf{right}(s)} y_{t,l} \leq 1- \sum_{j\in \textbf{C}(s)} z_{\textbf{V}(s),j},\ \forall t \in \{1,...,T\},\ s\in  \textbf{splits}(t)&\\ \notag
&\ \  z_{i,j}\leq z_{i,j+1},\ \forall i\in \{1,...,d\},\ j\in \{1,...,K_{i}-1\}&\\ \notag
&\ \  z_{i,j}  \in \{0,1\},\ \forall i\in\{1,...,d\},\ j\in \{1,...,K_{i}\} &\\ \notag
&\ \ y_{t,l} \geq 0,\ \forall t \in \{1,...,T\},\ l\in \textbf{leaves}(t) &\\ \notag
 &\ \  x_i \leq a_{i,j} +2(1-z_{i,j}) B ,\ \forall i\in\{1,...,d\},\ j\in \{1,...,K_{i}\} &\\ \notag
 &\ \  x_i > a_{i,j} -2z_{i,j} B,\ \forall i\in\{1,...,d\},\ j\in \{1,...,K_{i}\}  .& \notag
\end{align}
This formulation has a quadratic objective function and linear constraints. Similarly, we can formulate \eqref{eq:opt_ite} with $A \neq \emptyset$ by adding linear constraints $(a_i-\mu)^T\Sigma^{-1}(x-a_i)<0,\ \mbox{$\forall a_i \in A$}$ to \eqref{eq:forest}. Note that both the number of decision variables and the number of constraints are linearly dependent on the total number of nodes in the random forest.

\subsection{Tractable Formulation for Neural Network}\label{sec:neural}

A neural network $g(\cdot)$ is a network that connects a large number of computational units (known as neurons) \cite{goodfellow2016deep,chua1988cellular}. According to its task, a network has a specific architecture that usually involves multiple layers of neurons and different operations over the neurons. For simplification, here we consider layers with consecutive architecture and each layer of the neural network only contains one specific structure. 

The key part of the reformulation is to deal with the non-linearity brought by the maximum function. Our treatment of the maximum function follows from \cite{tjeng2017verifying}, which rewrites neural network structures into linear equations with binary variables. 

In order to obtain tractable formulation for the constraint $g(x)\geq \gamma$, we independently handle each single layer in $g(\cdot)$. Assume we have $l$ layers in $g(\cdot)$, where $g_i(\cdot)$ denotes the $i$th layer. Given input $x$, the output of the neural network can be represented as $g(x)=g^l(g^{l-1}(...g^1(x)))$. For convenience, we introduce $x_i$ to denote the output of the $i$th layer (note that it is also the input for the $i+1$th layer). In other words, for the $i$th layer we have $x_i=g^i(x_{k-1})$. Using these notations, we can transform the constraint $g(x)\geq \gamma$ into a sequence of constraints: \begin{align*}
        &\ \ x_{l} \geq \gamma, &\\ 
         &\ \  x_{l} = g^{l}(x_{l-1}), &\\
    &\ \  x_{l-1} = g^{l-1}(x_{l-2}), &\\ 
    &\ \  ...,\ &\\ 
    &\ \  x_{1} = g^1(x) .
    \end{align*}
This transformation makes clear that the constraints altogether are tractable if the constraint for each layer (i.e. $x_i=g^i(x_{k-1})$) is tractable 
. Note that both the number of decision variables and the number of constraints are linearly dependent on the total number of neurons in the neural network. In the rest of this section, we discuss the reformulation of neural network layers concerning different structures. 

\subsubsection{Fully Connected Layer}
In a fully connected layer, each neuron performs a linear transformation on the input. We consider a layer with $n$ neurons and the input for this layer is a vector $x\in \R^m$. We use $w_i\in \R^m$ and $b_i\in\R$ to denote the weight and bias respectively for the linear transformation in the $i$th neuron. Then the output of the $i$th neuron can be represented by $y_i=w_i^T x+b_i$. To summarize, the output of the layer, $y=[y_1;y_2;...;y_n]\in \R^n$, is given by
$$y={W}^T x + b,$$ where $W=[w_1,w_2,...,w_n]$ and $b=[b_1;b_2;...;b_n]$.

\subsubsection{ReLU Layer}
In a rectified linear unit (ReLU) layer, negative elements in the input are replaced by 0's. For the $i$th input, the output is given by $y_i=max\{x_i,0\}$. This can be represented by \begin{align*}
        &\ \  y_i \leq x_i - l(1-z_{i}), &\\ 
    &\ \  y_i\geq  x_i, &\\ 
    &\ \  y_i \leq u z_{i},\ &\\ 
    &\ \  y_i \geq 0,\  &\\ 
    &\ \  z_{i}  \in \{0,1\}\ ,
    \end{align*}
    where $z_i \in \{0,1\} $ is a binary variable, $u$ and $l$ are the upper and lower bounds of the input respectively. 

\subsubsection{Normalization Layer}
In a normalization layer, the input is normalized and linearly transformed to make the gradient decent algorithm more efficient. Again we assume the input is $x\in\R^m$ with a given normalization parameter $\mu \in \R^m$ and $\Sigma\in \R^{m\times m}$. Moreover, we have the transformation matrix $\gamma \in \R^{m\times m}$ and bias vector $\beta\in \R^m$. The output is given by $$y=\gamma \left( \Sigma^{-1/2} (x-\mu)  \right) +\beta. $$

\subsubsection{Pooling Layer}
In a pooling layer, a ``filter'' that can be applied to adjacent elements in a vector or matrix goes through the input with a certain stride. Such type of layer is used to summarize ``local'' information and reduce the dimension of the input. Max pooling and average pooling are two types of commonly used filters.

Suppose the input is represented by matrix $x\in \R^{m_1\times m_2} $, where $x_{ij}$ denotes the element in the $i$th row $j$th column. The size of the filter is $s_1\times s_2$ with stride $(s_1,s_2)$. The output have size $y\in\R^{n_1,n_2}$, where $n_1=m_1/s_1$ and $n_2=m_2/s_2$. We assume that the value of $s_1,s_2$ are carefully chosen so that $n_1$ and $n_2$ are integers.

For average pooling layer, we have $$y_{ij}=\frac{\sum_{r=(i-1)s_1+1}^{is_1}\sum_{c=(j-1)s_2+1}^{js_2}x_{rc}}{s_1s_2}$$
for $i=1,...,n_1$, $j=1,...,n_2$.

For max pooling layer, we have $y_{ij}=\max_{(r,c)\in S }x_{rc}$ for $i=1,...,n_1$, $j=1,...,n_2$, where $S=\{(r,c)|r=(i-1)s_1+1,...,is_1,c=(j-1)s_2+1,...,js_2\}$. The tractable formulation is given by\begin{align*}
        &\ \  y_{ij} \leq x_{rc} - (u-l)(1-z_{rc}), & (r,c)\in S\\ 
    &\ \  y_{ij}\geq  x_{rc}, & (r,c)\in S\\ 
     &\ \ \sum_{(r,c)\in S}  z_{rc}=1\\ 
    &\ \  z_{rc}  \in \{0,1\}, & (r,c)\in S.
    \end{align*}

\subsubsection{Convolutional Layer}
In a convolutional layer, several filters are used to extract features from the input. The input of the layer is $x\in \R^{m_1,m_2}$. Suppose we have $r$ filters and assume the filters have size $s_1\times s_2$ with stride $(t_1,t_2)$. We use $w_i \in \R^{t_1 t_2}$ and $b_i \in \R^{t_1 t_2}$ to denote the weight and bias for the $i$th filter. The output is $y\in \R^{n_1\times n_2 \times r}$, where $n_1=(m_1-s_1)/t_1$ and $n_2=(m_2-s_2)/t_2$. Again we assume the numbers are carefully chosen so that $n_1,n_2$ are integers. 

Then we have \begin{align*}
& y_{ijk}=w_k^T (\tilde{x}_{ij})+b_k,\\
& \tilde{x}_{ij}=[x_{(i-1)t_1+1,(j-1)t_2+1};x_{(i-1)t_1+2,(j-1)t_2+1};...;x_{(i-1)t_1+1,(j-1)t_2+2},...;x_{(i-1)t_1+s_1,(j-1)t_2+s_2}].
\end{align*}
for integers $1\leq i \leq n_1$, $1\leq j \leq n_2$ and $1\leq k \leq r$.

\subsubsection{Reformulation in the Output Layer}
Here we discuss the reformulation of the output layer, which also provides us clues on how other more general problems in classification tasks are potentially transformable into the constraint $g(x)\geq \gamma$. Although the output layer is usually highly nonlinear, we show how to formulate it as linear mixed-integer constraints. 

In classification tasks, the neural network usually uses a softmax layer as the output layer for training purposes. Suppose the classification problem has $n$ categories in total, the last layer inputs $x\in \R^n $ and outputs $y\in\R^n$ with $y_i=\frac{e^{x_i}}{\sum_{j=1}^{n} e^{x_j}}$. The prediction for classification is determined by the maximum value of $y_i$. Indeed, the result is equivalent if we determine the categories by the maximum value of $x_i$.

When the constraint is $g(X)=i$ or $g(X)\neq i$, we can use this equivalence to reformulate the last layer (and therefore complete the formulation for the whole network). Specifically, $g(X)=i$ can be formulated as $x_i\geq x_j,for\  j\neq i$ and $g(X)\neq i$ can be formulated as $x_i \leq \max_{j\neq i} {x_j}$, where $j\neq i$ denotes $j$ is an element for the set that contains all possible indexes except $i$. For tractable form, the latter formula can be further rewritten as:
\begin{align*}
        &\ \  x_i\leq x_j + (1-z_j)(u-l),\  j \neq i. &\\  \notag
        &\ \  \sum_{j\neq i} z_{j} \geq 1, &\\  \notag
        &\ \  z_{j}  \in \{0,1\},\ i\neq c.  \notag
\end{align*}

\section{Experiments}
\label{sec:experiments}

This section presents several experimental results using our Algorithm \ref{algo:main} for neural network and random forest predictors. In Section \ref{sec:toy}, we consider two simple toy examples. The first problem has one dominating point and the second problem has multiple dominating points. To illustrate the efficiency of the IS scheme, we compare it with the naive use of a uniform IS estimator. In Section \ref{sec:magic}, we consider a realistic problem generated from a classification data set with a high dimensional feature space. 

\subsection{Toy Problems}
\label{sec:toy}

Consider a problem where $X $ follows a distribution $f(x)$, and the set $\{x:g(x)\geq\gamma\}$ is known to lie inside $[l,u]^d$ where $d$ is the dimension of the input variable $X$. The uniform IS estimator is given by\begin{equation*}
Z_{uniform}= I(g(X) \geq \gamma) f(X) (u-l)^d,
\end{equation*}
where $X$ is generated from a uniform distribution on $[l,u]^d$. This estimator has a polynomially growing relative efficiency as the magnitude of the dominating points grows \cite{huang2018rare}, but the efficiency also depends significantly on the size of the bounded set, i.e., $l,u,d$. 

The first problem has input $x=[x_1,x_2]$ over the bounded space $[0,5]^2$. We generate 2,601 samples using a uniform grid over the space with a mesh of 0.1 on each coordinate and use the function \begin{equation} \label{eq:easy_function}
y(x)=(x_1-5)^3+(x_2-4.5)^3+(x_1-1)^2+x_2^2+500
\end{equation}
to label these samples. The dataset we obtained is denoted as $D=\{(X_n,Y_n)\}$. $g(x)$ is trained using $D$. We consider only $X$ in the region $[0,5]^2$, so that $g(x)$ can be thought of as being set to 0 outside this box. We use $\gamma =500$ in this example and the shape of the rare-event set $\{x:g(x) \geq \gamma\}$.

We first train a random forest $g(x)$, which ensembles three regression trees. The three regression trees are averaged and each of them has around 600 nodes. The rare-event set is presented in Figure \ref{fig:exp1_set_rf}. The dominating point is obtained by implementing Algorithm \ref{algo:main}, which is located at $(3.05,2.65)$. 
We recall the problem setting that the input $X$ follows a Gaussian distribution. In particular, we use Gaussian distributions $N(0,I \sigma^2)$, where $I$ denotes the identity matrix and $\sigma^2 \in \mathbb{R}^+$. In our experiment, we vary the value of $\sigma^2$ to create problems with different rarity, where a smaller $\sigma^2$ gives a rarer probability.

Figures \ref{fig:exp1_mean_rf} and \ref{fig:exp1_ci_rf} present the experimental results based on 50,000 samples. In Figure \ref{fig:exp1_mean_rf}, we observe that the estimates for the two IS schemes are similar in all considered cases. On the other hand, Figure \ref{fig:exp1_ci_rf} shows the relative error for the proposed IS is smaller in all $\sigma^2$ considered. Moreover, as the rarity increases, the relative error of the proposed IS increases from roughly 2.5 to 5, whereas the relative error of the uniform IS increases from 5 to 40. The slower increasing rate indicates that the proposed IS scheme is more efficient and the outperformance is stronger for rarer problems. 

\begin{figure}[t]
							\centering
							
							\begin{minipage}[b]{0.4\textwidth}
								\centering
								\includegraphics[width=\linewidth]{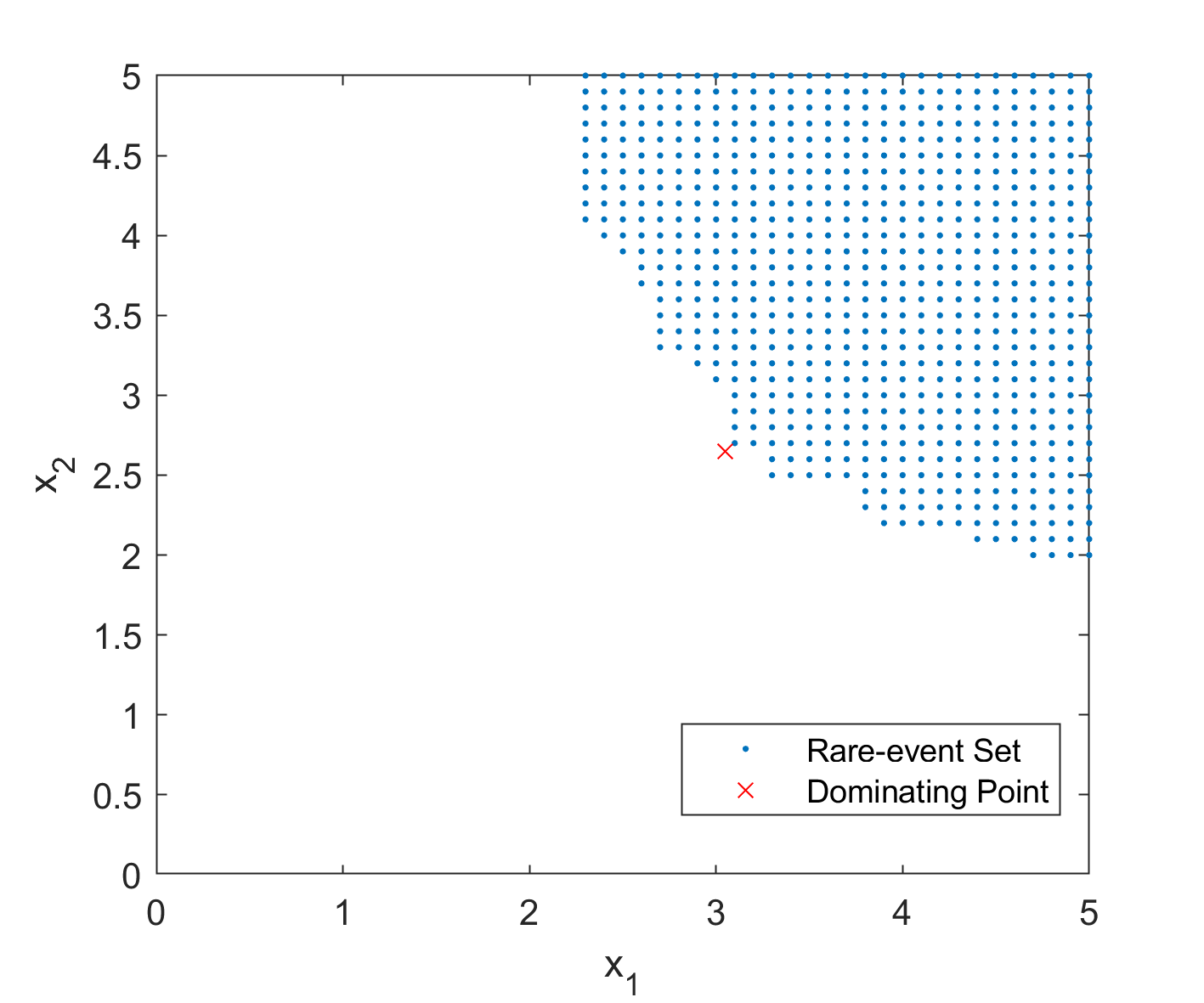}
								\caption{Rare-event set and dominating points for the random forest (case 1).}
								\label{fig:exp1_set_rf}
							\end{minipage}
                            \begin{minipage}[b]{0.4\textwidth}
								\centering
                            \includegraphics[width=\linewidth]{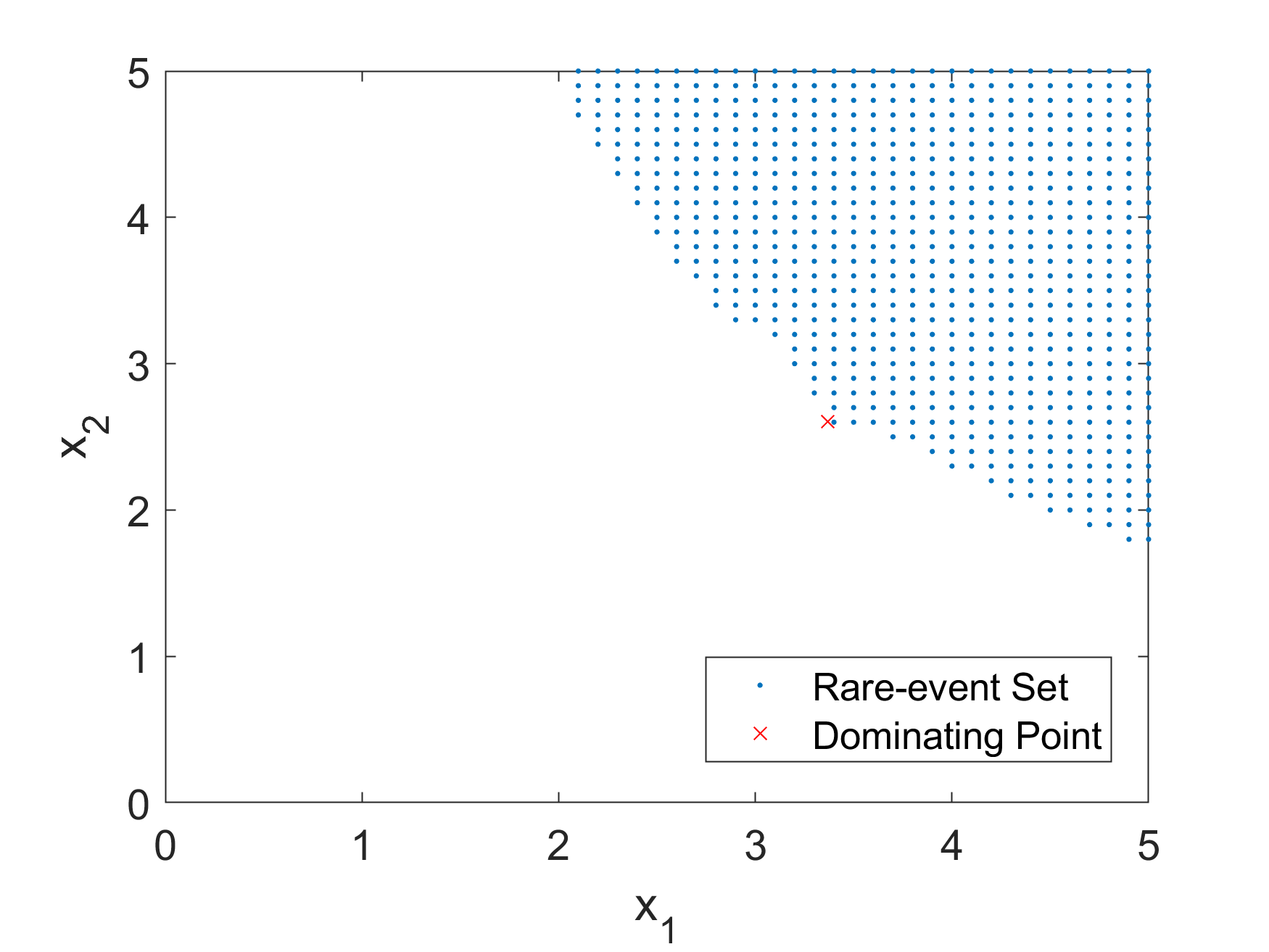}
	\caption{Rare-event set and dominating points for the neural network (case 1).}
	\label{fig:exp1_set_nn}
							\end{minipage}
			\end{figure}

\begin{figure}[t]
							\centering
							
							\begin{minipage}[b]{0.4\textwidth}
								\centering
								\includegraphics[width=\linewidth]{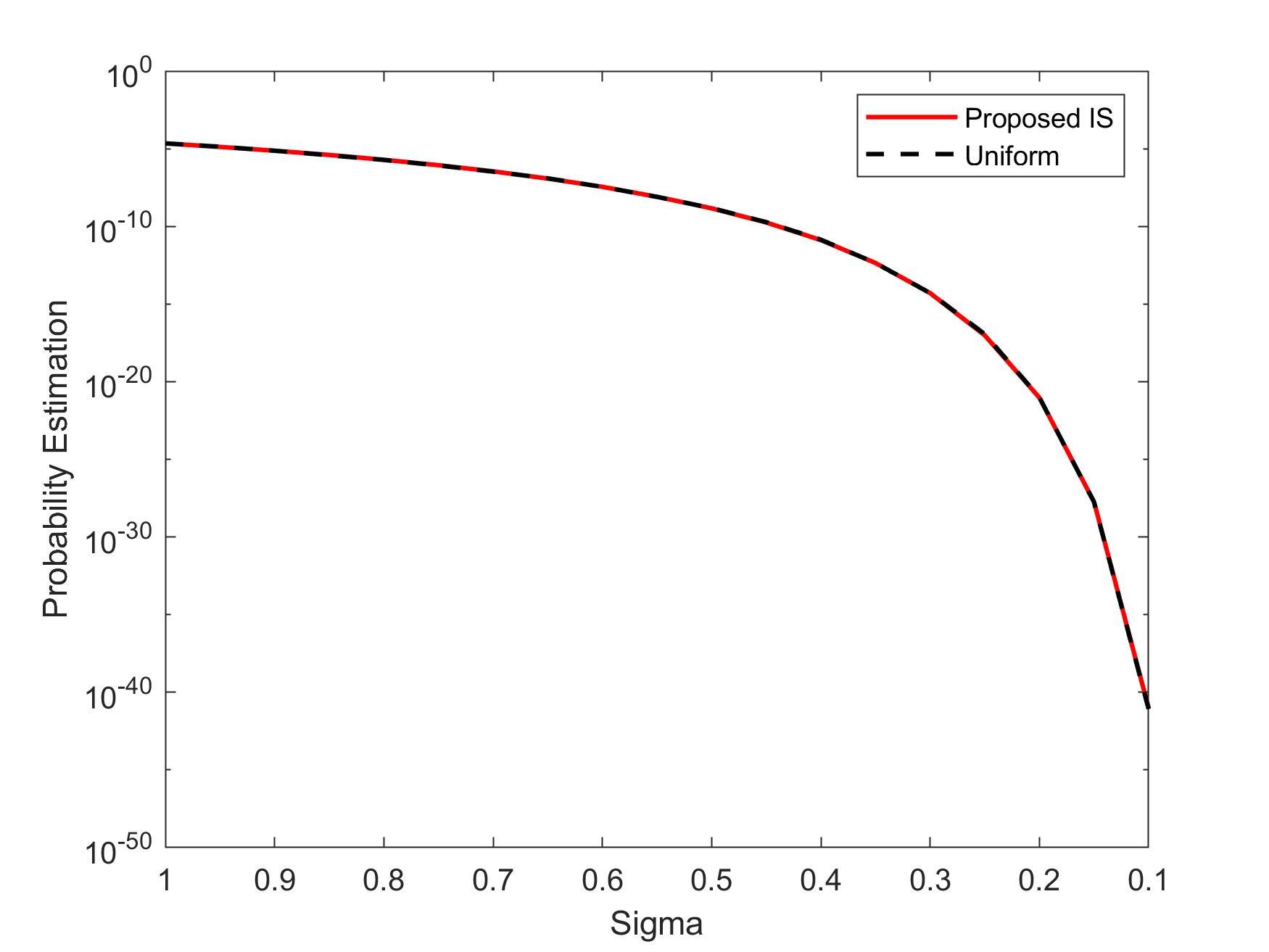}
								\caption{Probability estimation with different numbers of samples. Random forest, case 1.}
								\label{fig:exp1_mean_rf}
							\end{minipage}
                            \begin{minipage}[b]{0.4\textwidth}
								\centering
                            \includegraphics[width=\linewidth]{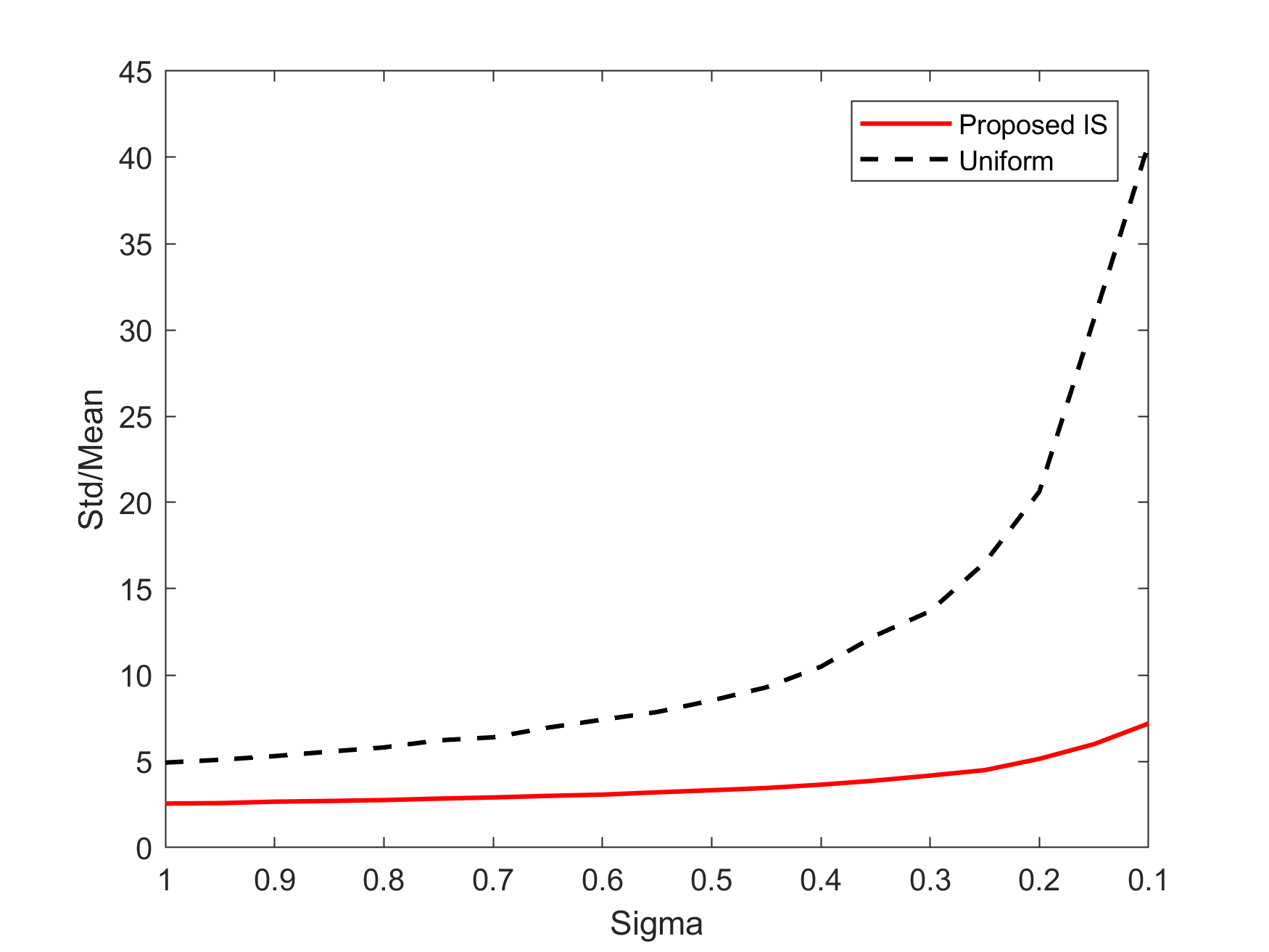}
	\caption{95\% confidence interval half-width with different numbers of samples. Random forest, case 1.}
	\label{fig:exp1_ci_rf}
							\end{minipage}
			\end{figure}
Next, we train a neural network predictor as $g(x)$. The neural network has 3 layers with 100 neurons in each of the 2 hidden layers, and all neurons are ReLU. The defined rare-event set is presented in Figure \ref{fig:exp1_set_nn}. We observe that the set is roughly convex and should have a single dominating point. We obtain the dominating point for the set at $(3.3676,2.6051)$. Figures \ref{fig:exp1_mean_nn} and \ref{fig:exp1_ci_nn} shows our results. Again we observe the proposed IS scheme provides smaller relative errors in all cases and the advantage increases with the rarity level (the relative error increases from 2.5 to 10 for the proposed IS and 5 to 55 for the uniform IS).

\begin{figure}[t]
							\centering
							
							\begin{minipage}[b]{0.4\textwidth}
								\centering
								\includegraphics[width=\linewidth]{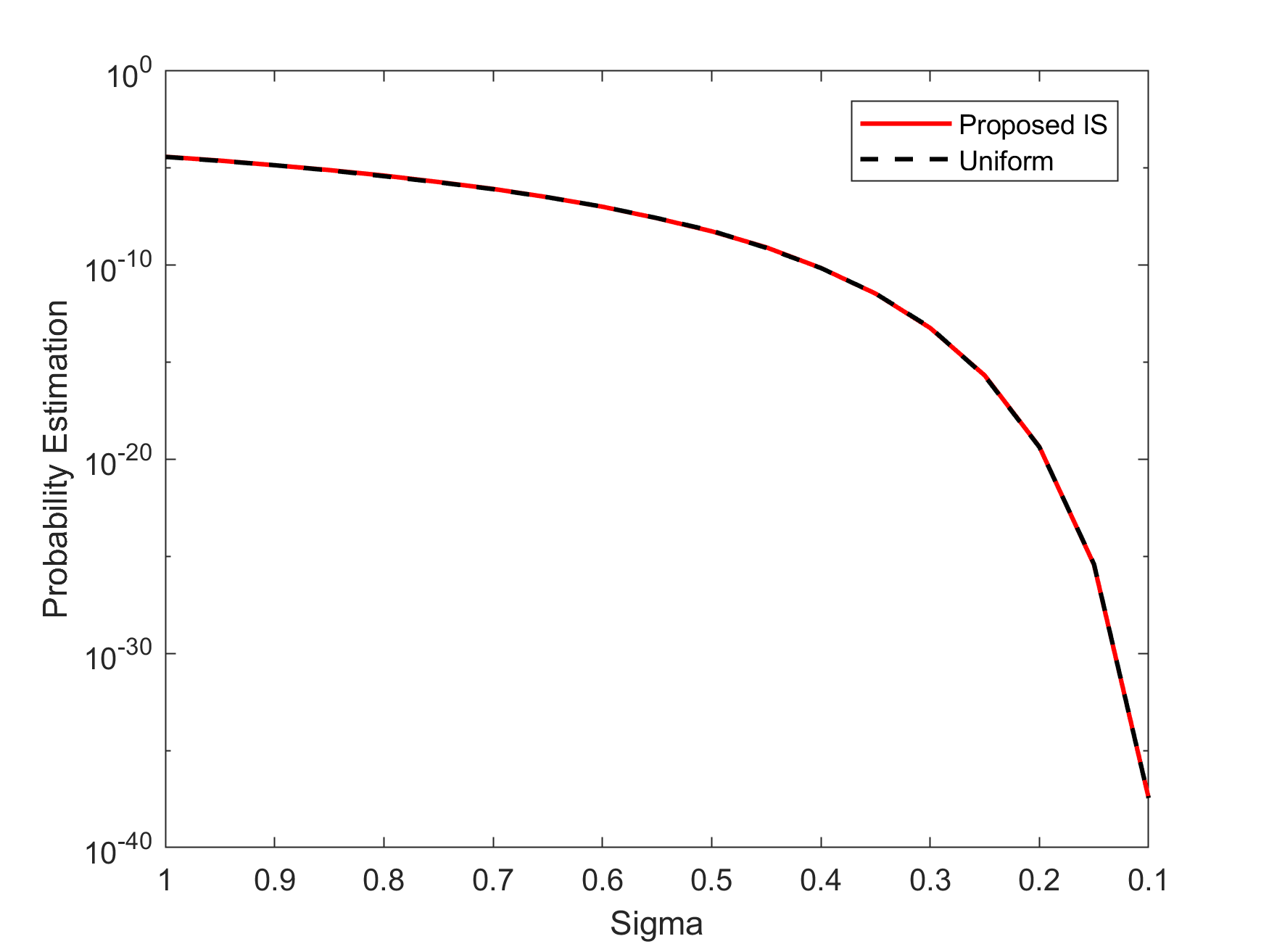}
								\caption{Probability estimation with different numbers of samples. Neural network, case 1.}
								\label{fig:exp1_mean_nn}
							\end{minipage}
                            \begin{minipage}[b]{0.4\textwidth}
								\centering
                            \includegraphics[width=\linewidth]{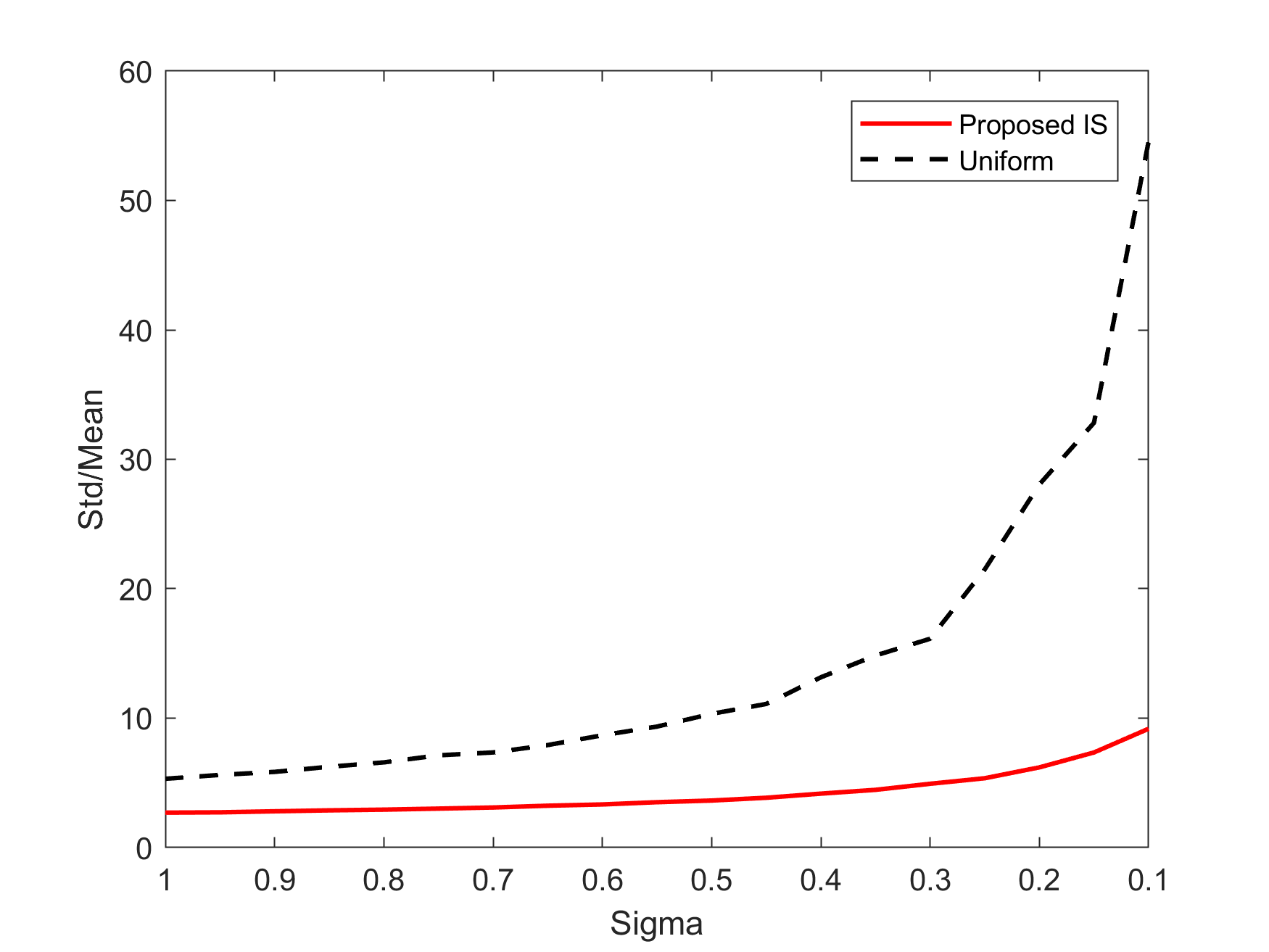}
	\caption{95\% confidence interval half-width with different numbers of samples. Neural network, case 1.}
	\label{fig:exp1_ci_nn}
							\end{minipage}
			\end{figure}

Next, we consider true output values generated according to the function \begin{equation} \label{eq:two_peak_function}
y(x)=10 \times e^{-\left( \frac{x_1-5}{3} \right) ^2 -\left( \frac{x_2-5}{4} \right) ^2} + 10 \times e^{- {x_1}  ^2 -\left( {x_2-4.5} \right) ^2}. 
\end{equation} Again we use a uniform grid over $[0,5]^2$ with a mesh of 0.1 on each coordinate to train the predictors. The random forest ensembles three regression trees with around 600 nodes and the neural network with 2 hidden layers, 100 neurons in the first hidden layer and 50 neurons in the second hidden layer. All neurons in the neural network are ReLU. We set $\gamma =8$. The shapes of the rare-event sets are shown in Figures \ref{fig:exp2_set_rf} and \ref{fig:exp2_set_nn}. We observe that the set now consists of two disjoint regions and therefore we expect to obtain multiple dominating points. Using Algorithm \ref{algo:main}, we obtain two dominating points in each case: $(0,4.15)$ and $(3.75,3.55)$  for the random forest model;  $(0.113,	4.162)$ and $(4.187,3.587)$ for the neural network model. We use these dominating points to construct a mixture distribution, as discussed in Section \ref{sec:problem_setting}, as the IS distribution. Again we vary $\sigma^2$ to obtain problems with different rarities and use 50,000 samples for each case.

\begin{figure}[t]
							\centering
							
							\begin{minipage}[b]{0.4\textwidth}
								\centering
								\includegraphics[width=\linewidth]{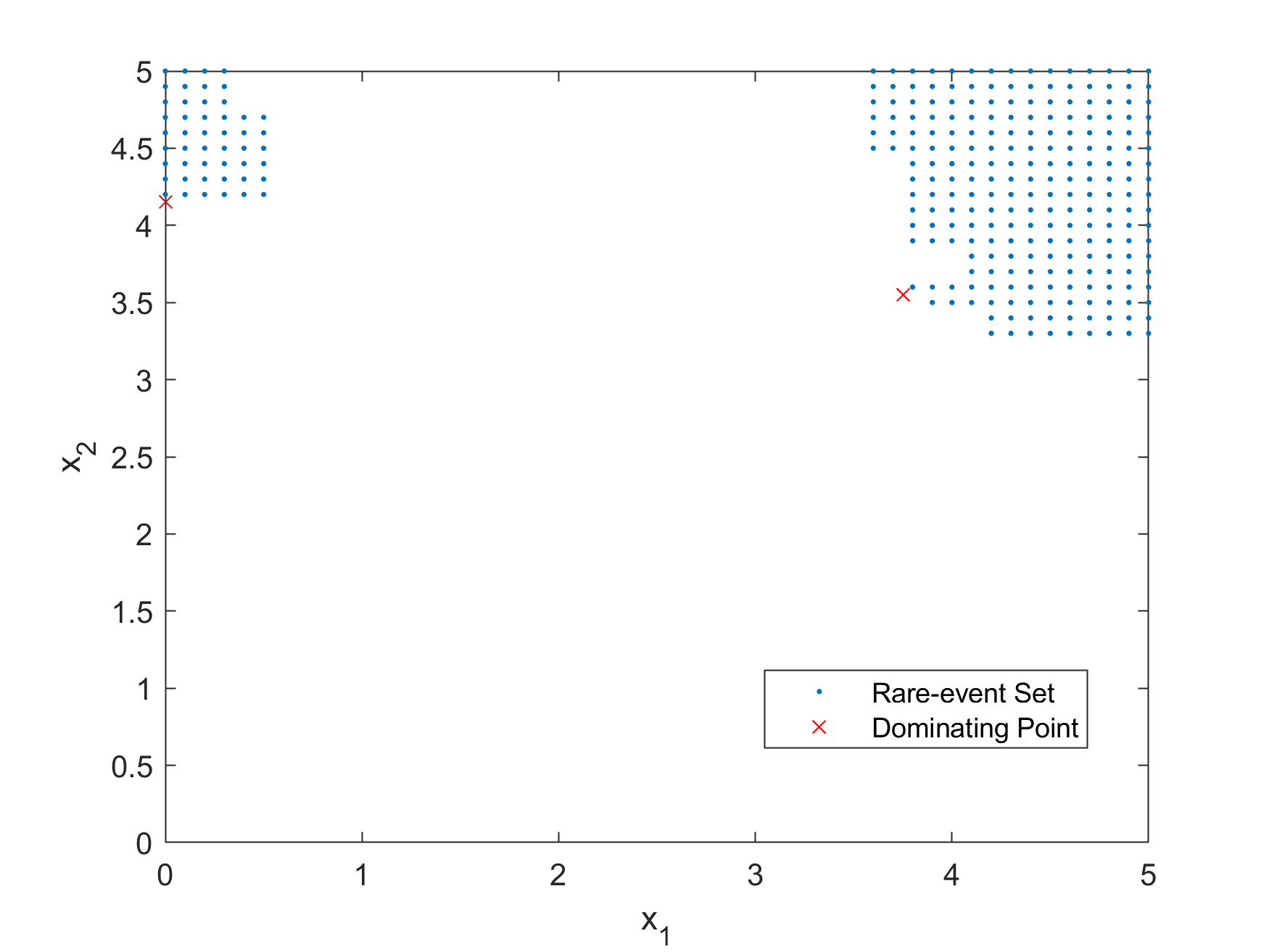}
								\caption{Rare-event set and dominating point for the random forest (case 2).}
								\label{fig:exp2_set_rf}
							\end{minipage}
                            \begin{minipage}[b]{0.4\textwidth}
								\centering
                            \includegraphics[width=\linewidth]{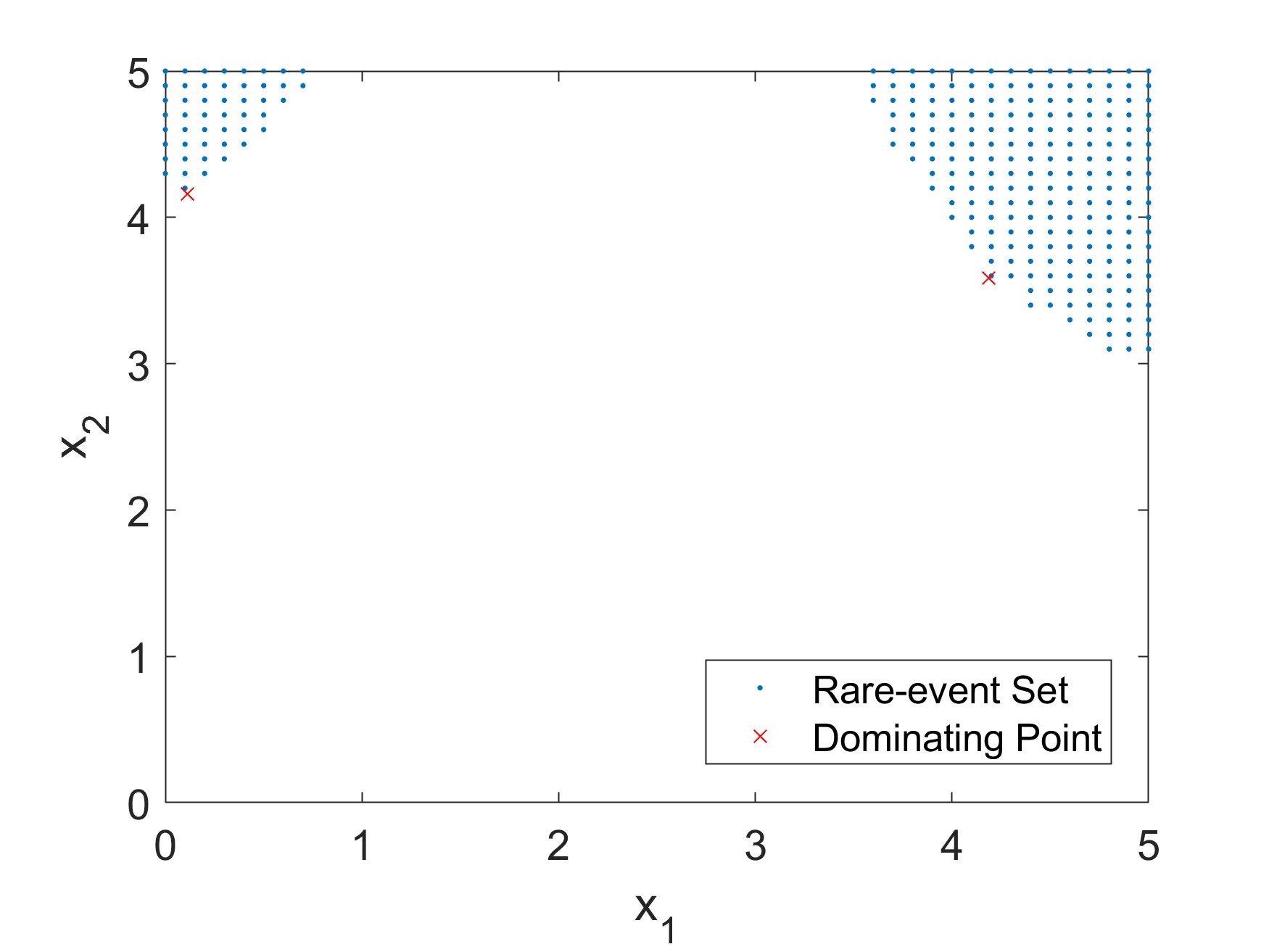}
	\caption{Rare-event set and dominating point for the neural network (case 2).}
	\label{fig:exp2_set_nn}
							\end{minipage}
			\end{figure}

The experimental results for the random forest predictor are shown in Figures \ref{fig:exp2_mean_rf} and \ref{fig:exp2_ci_rf}, and the results for the neural network predictor are shown in Figures \ref{fig:exp2_mean_nn} and \ref{fig:exp2_ci_nn}. Similar to the previous problem, both IS schemes give similar estimates in all the cases, as observed in Figures \ref{fig:exp2_mean_rf} and \ref{fig:exp2_mean_nn}. The relative errors shown in Figures \ref{fig:exp2_ci_rf} and \ref{fig:exp2_ci_nn} illustrate that, as the probability of interest decreases, the relative error ratio between the uniform IS and the proposed IS increases from 2 to around 5-6. We can conclude that the proposed IS scheme again outperforms the uniform IS and is more preferable as the rarity increases.

\begin{figure}[t]
							\centering
							
							\begin{minipage}[b]{0.4\textwidth}
								\centering
								\includegraphics[width=\linewidth]{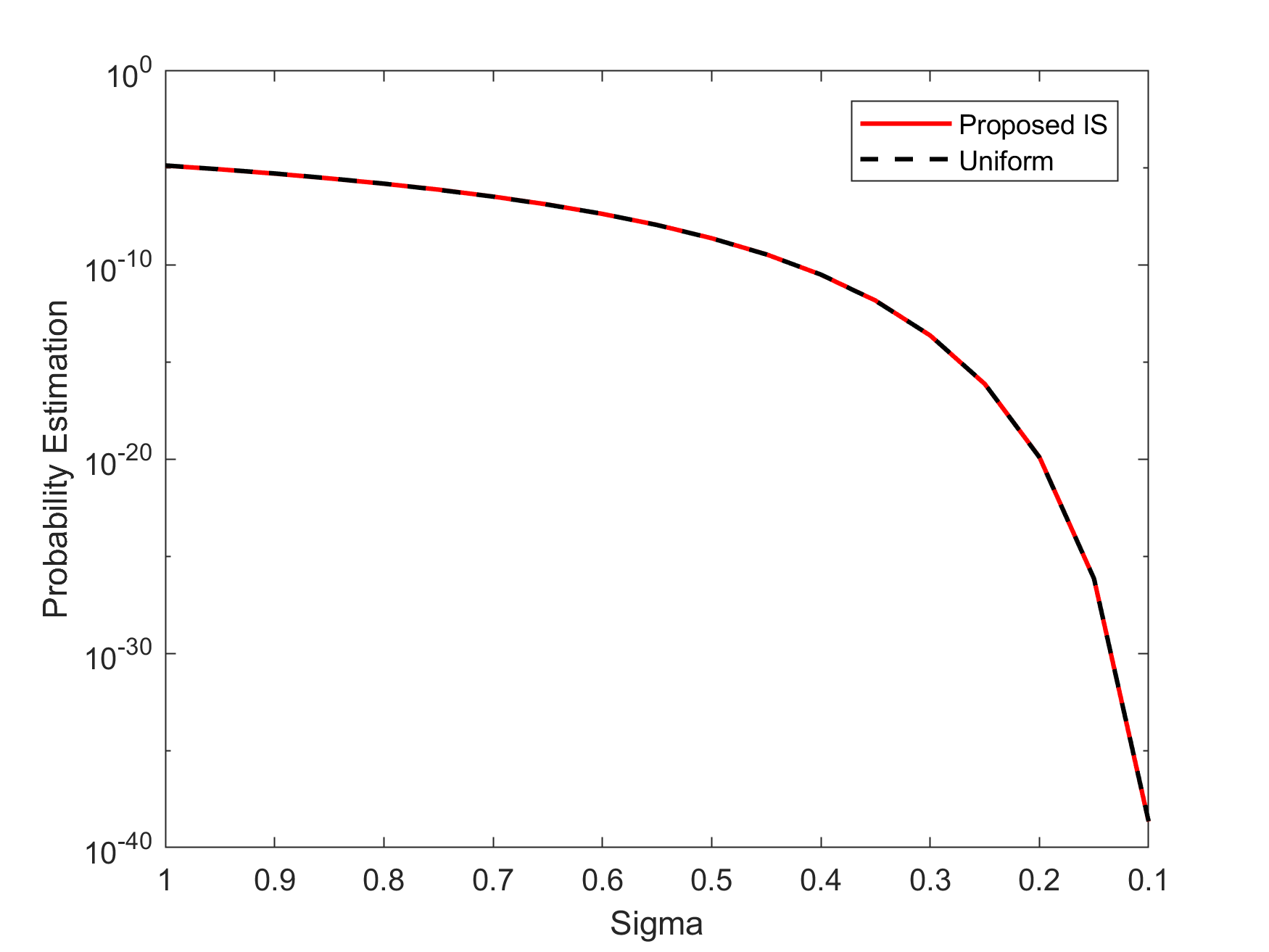}
								\caption{Probability estimation with different numbers of samples. Random forest, case 2.}
								\label{fig:exp2_mean_rf}
							\end{minipage}
                            \begin{minipage}[b]{0.4\textwidth}
								\centering
                            \includegraphics[width=\linewidth]{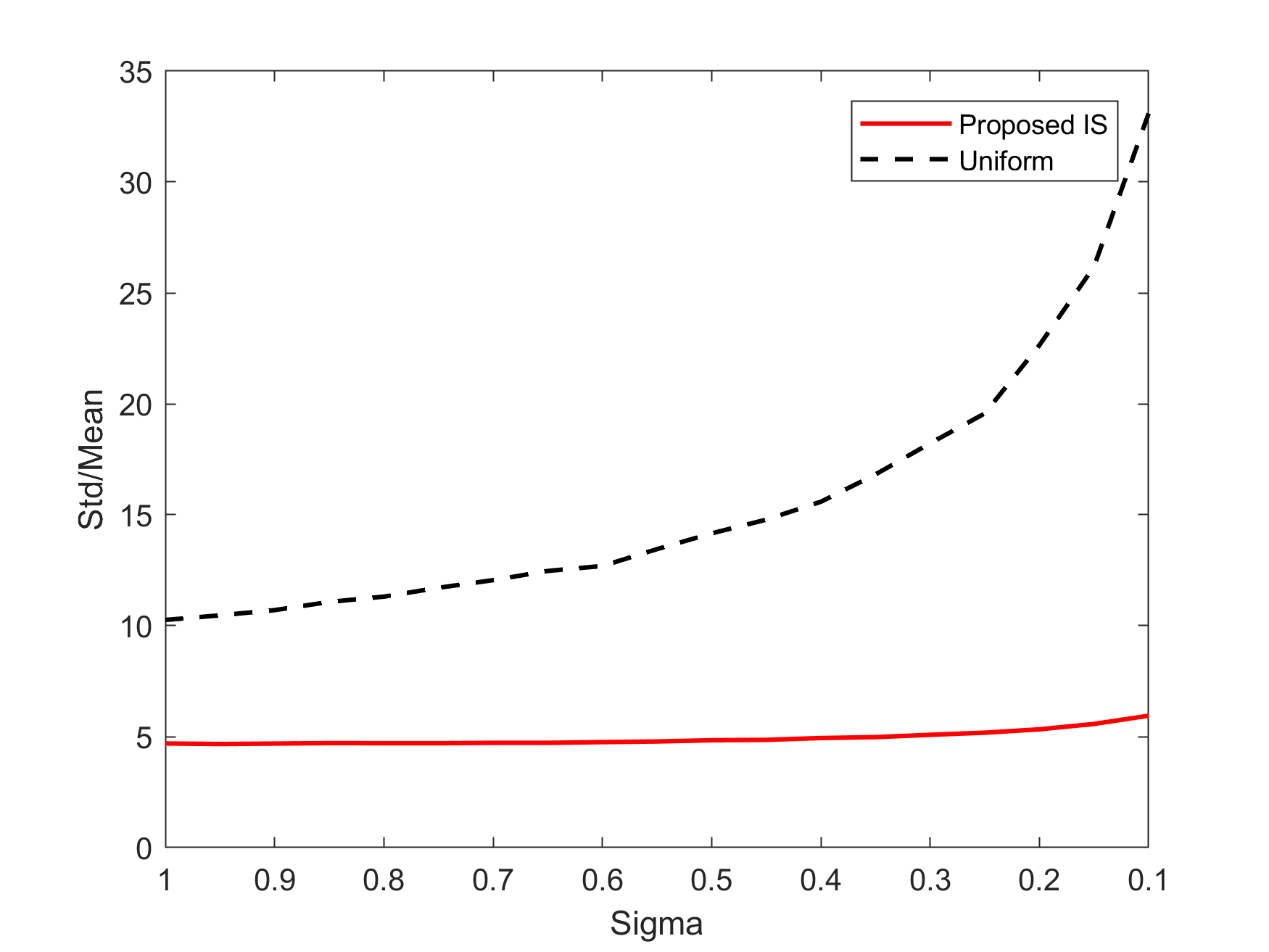}
	\caption{95\% confidence interval half-width with different numbers of samples. Random forest, case 2.}
	\label{fig:exp2_ci_rf}
							\end{minipage}
			\end{figure}			
			
\begin{figure}[t]
							\centering
							
							\begin{minipage}[b]{0.4\textwidth}
								\centering
								\includegraphics[width=\linewidth]{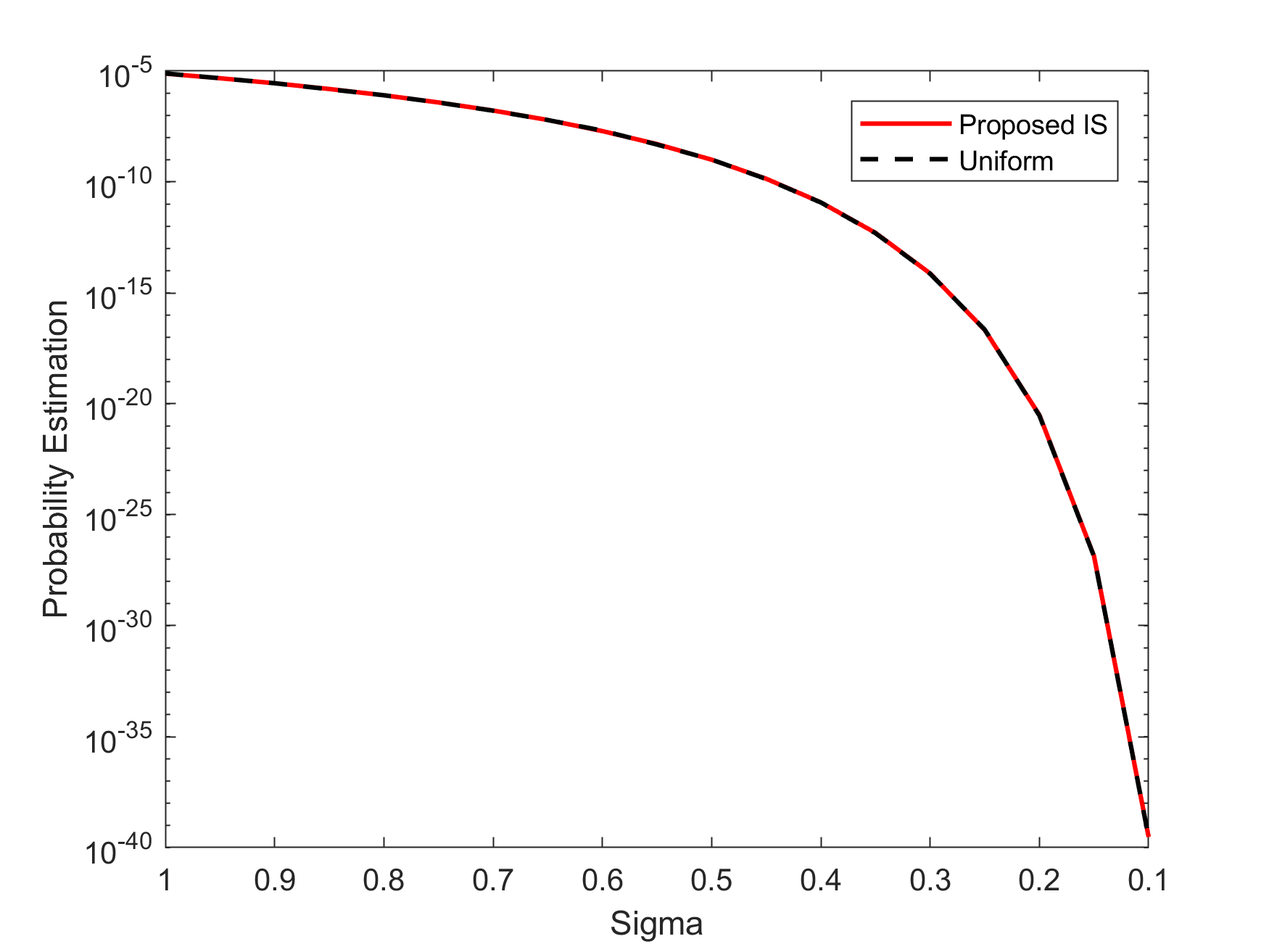}
								\caption{Probability estimation with different numbers of samples. Neural network, case 2.}
								\label{fig:exp2_mean_nn}
							\end{minipage}
                            \begin{minipage}[b]{0.4\textwidth}
								\centering
                            \includegraphics[width=\linewidth]{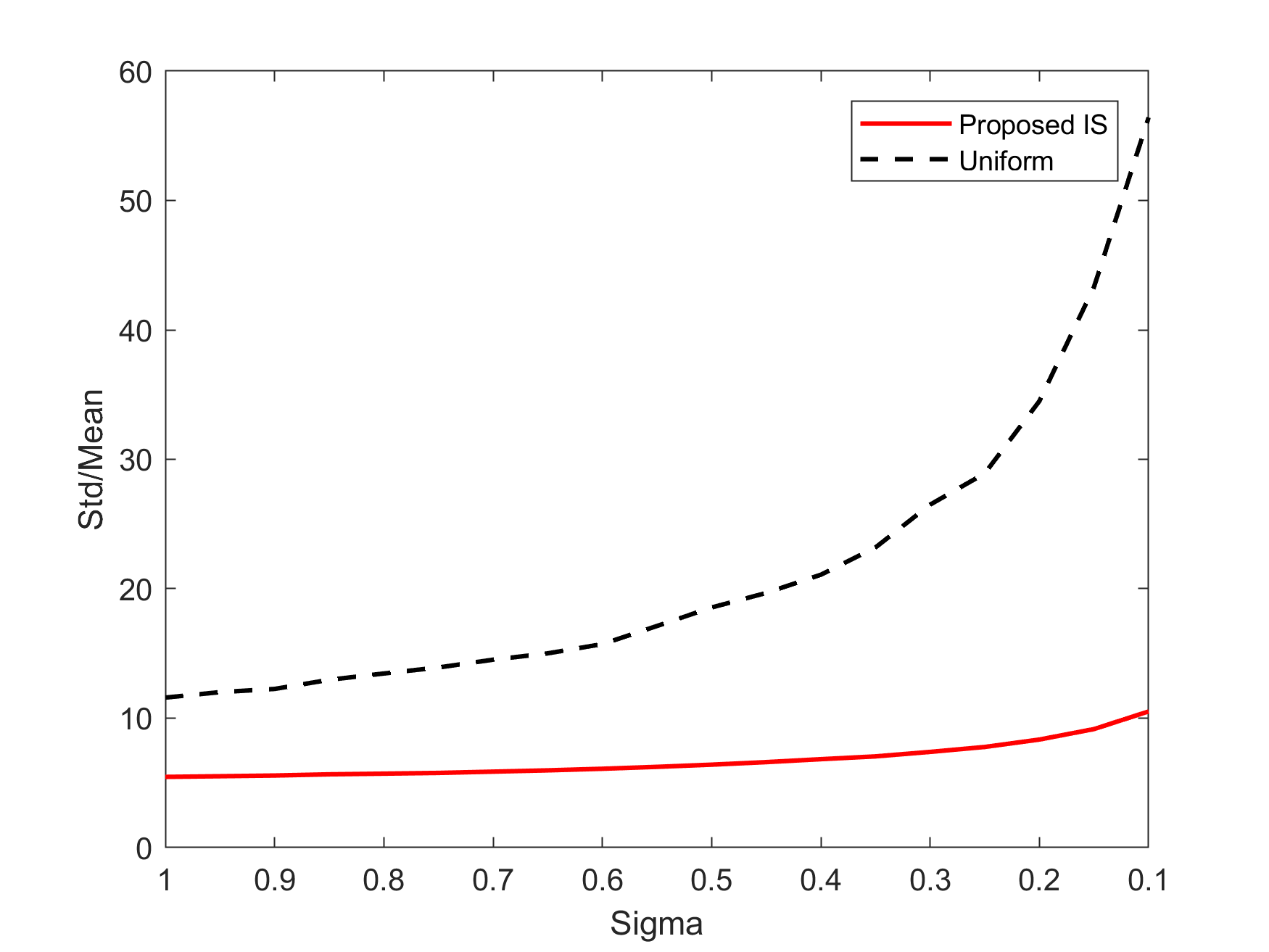}
	\caption{95\% confidence interval half-width with different numbers of samples. Neural network, case 2.}
	\label{fig:exp2_ci_nn}
							\end{minipage}
			\end{figure}

\subsection{MAGIC Gamma Telescope Data Set}
\label{sec:magic}

We study a rare-event probability estimation problem from a realistic classification task. The classification problem uses the MAGIC Gamma Telescope data set in the UCI Machine Learning Repository \cite{asuncion2007uci}. The problem is to classify images of electromagnetic showers collected by a ground-based atmospheric Cherenkov gamma telescope. The features of the data are 10-dimensional characteristic parameters of the images and the data set contains 19020 data points in total. Studies \cite{bock2004methods,savicky2004experimental,dvovrak2007softening} use machine learning predictors to discriminate images caused by a ``signal'' (primary gammas) from those initiated by the ``background'' (cosmic rays in the upper atmosphere).

To train the predictors, we allocate 15,000 data points as the training set and use the remaining 4,020 data points as the testing set. We train a random forest that ensembles 10 random trees to achieve 85.6\% testing set accuracy. For neural network, we use 2 hidden layers with 20 neurons and achieved 87\% testing set accuracy. 

The rare-event probability of interest is the statistical robustness metric (Example \ref{example:robust_metric}) of the two trained predictors. Specifically, we consider a testing data point, say with input $x$ and true label $y$, that is correctly predicted in both predictors (the predicted value $g(x)$ is consistent with $y$). Then we perturb the input $x$ with a Gaussian noise $\epsilon \sim N(0,I \sigma^2)$ and estimate the probability of $P(g(x+\epsilon )\neq y ) $, where we vary the value of $\sigma^2$ to construct rare-event with different rarities. Note that, as discussed in Example \ref{example:robust_metric}, $P(g(x+\epsilon )\neq y ) $ can be transformed into the format considered in this paper, i.e. $P(g(X) > \gamma ) $.

First, we implement Algorithm \ref{algo:main} to obtain dominating points for the rare-event sets $\{g(x+\epsilon )\neq y \}$ with random forest and neural network as $g(\cdot)$  respectively. We obtain 53 dominating points for the rare-event sets associated with the random forest predictor and 217 dominating points in the neural network case. The IS distributions are constructed using these dominating points. In both problems, $\sigma^2$ ranges from 0.03 to 0.1 and we use 50,000 samples to estimate each target rare-event probabilities. 

The experimental results for the random forest and neural network are presented in Figures \ref{fig:exp_mean_rf_magic} and \ref{fig:exp_mean_nn_magic} respectively. We observe that the estimates are very accurate in all experiments (with different rarities), which are indicated by the tight 95\% confidence intervals. These results show that our proposed IS scheme performs well with large numbers of dominating points and in relatively high-dimensional problems.

\begin{figure}[t]
							\centering
							
							\begin{minipage}[b]{0.4\textwidth}
								\centering
								\includegraphics[width=\linewidth]{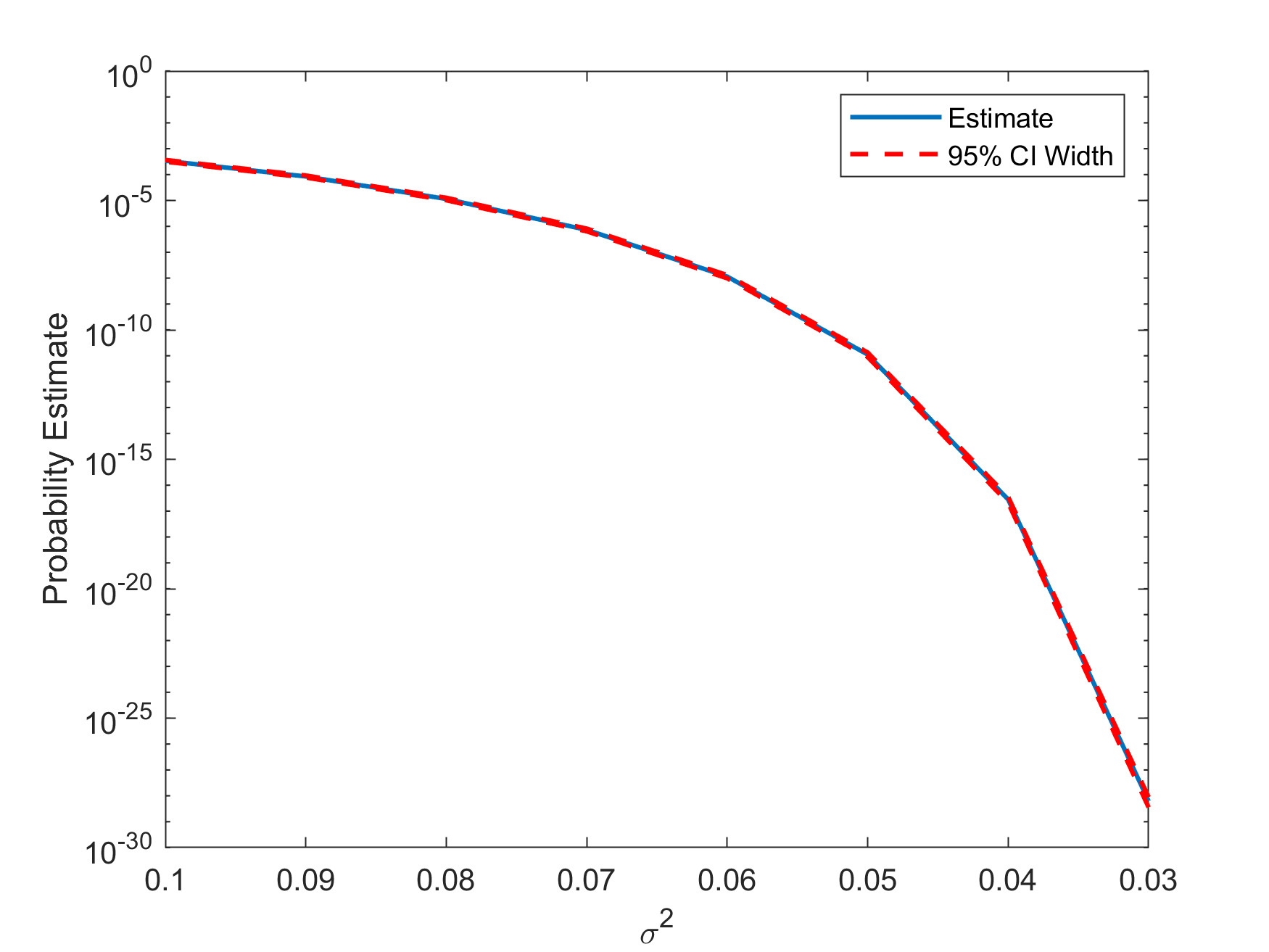}
								\caption{Probability estimation with different numbers of samples. Random forest, MAGIC.}
								\label{fig:exp_mean_rf_magic}
							\end{minipage}
                            \begin{minipage}[b]{0.4\textwidth}
								\centering
                            \includegraphics[width=\linewidth]{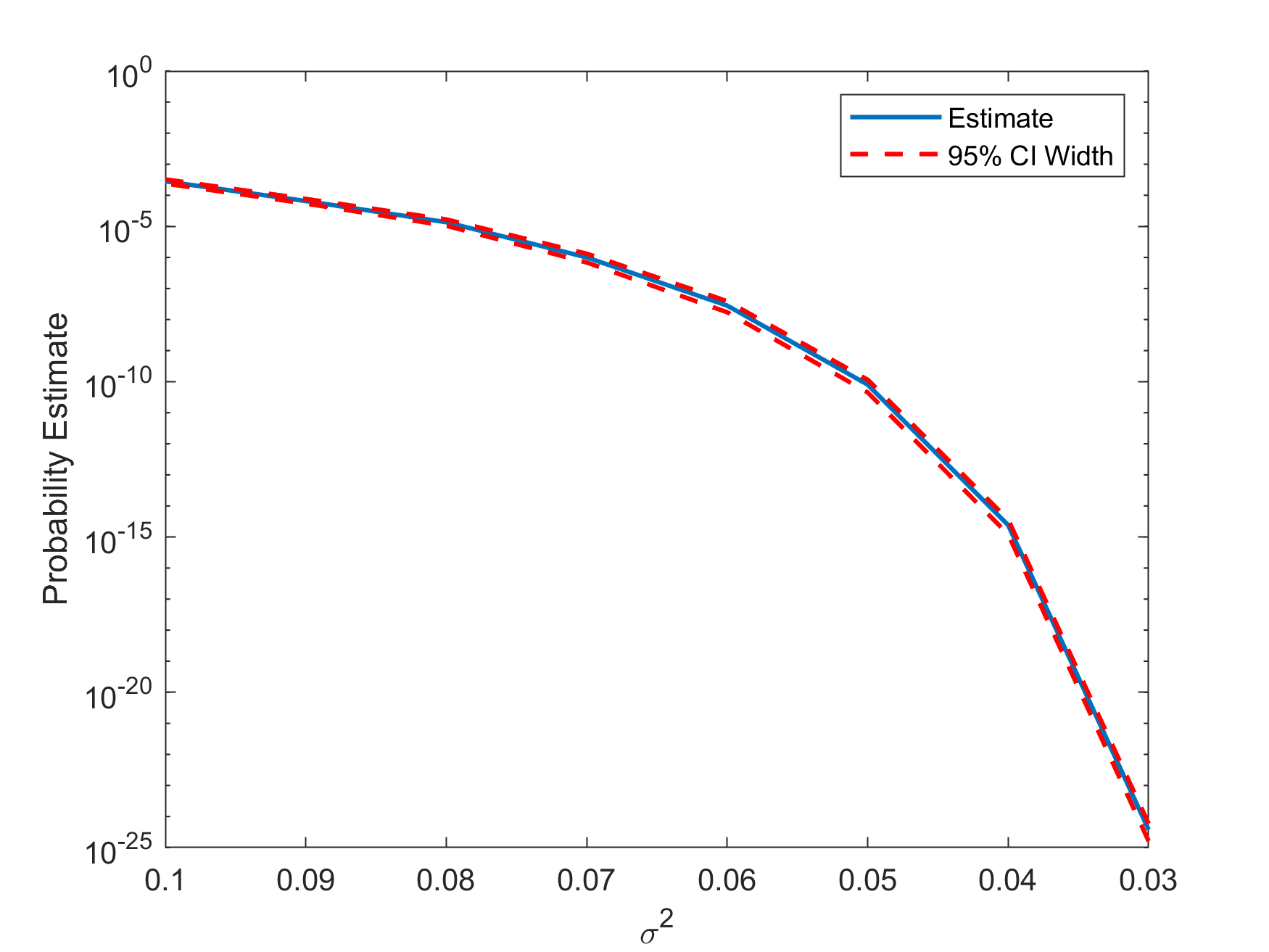}
	\caption{95\% confidence interval half-width with different numbers of samples. Random forest, MAGIC.}
	\label{fig:exp_mean_nn_magic}
							\end{minipage}
			\end{figure}

\section{Proof of Theorems}
\label{sec:guarantee}
Throughout this section, we write $f_1(\gamma)\sim f_2(\gamma)$ if $\lim_{\gamma\rightarrow\infty}f_1(\gamma)/f_2(\gamma)=1$. First of all, we adapt Theorem 4.1 in \cite{Hashorva2003gaussian} to obtain the following lemma. 

\begin{lemma}
Let $Y$ be a $d$-dimensional Gaussian random vector with zero mean and positie definite covariance matrix $\tilde{\Sigma}$. Suppose that $\tilde s=\tilde s(\gamma)$ is a $d$-dimensional vector such that as $\gamma\to\infty$, at least one of its components goes to $\infty$. Use $y^*$ to denote $\arg\min_{y\geq \tilde s}y^{T}\tilde{\Sigma}^{-1}y$. Then by Proposition 2.1 in \cite{Hashorva2003gaussian}, we know that there exists a unique set $I\subset\{1,\cdots,d\}$ such that 
\begin{subequations}
    \begin{align}
        &1\le |I|\le d; \label{Za}\\
        &y_I^*=\tilde s_I\ne \mathbf{0}_I;\label{Zb}\\
        &\text{If }J:=\{1,\dots,d\}\setminus I\ne\emptyset,\text{ then }y_J^*=-(\tilde{\Sigma}^{-1})_{JJ}^{-1}(\tilde{\Sigma}^{-1})_{JI}\tilde s_I\ge \tilde s_J\label{Zc}\\
        &\forall i\in I,e_i^{T}(\tilde{\Sigma}_{II})^{-1}\tilde s_I>0;\label{Zd}\\
        &\min_{y\ge t}y^{T}\tilde{\Sigma}^{-1}y=(y^*)^{T}\tilde{\Sigma}^{-1}y^*>0.\label{Ze}
    \end{align}
    \label{eqn:indexset}
\end{subequations}
We suppose that for sufficiently large $\gamma$, the set $I$ does not change with $\gamma$ and $\lim_{\gamma\to\infty}(\tilde s-y^*)_J=\tilde s_J^*$. Suppose further that $\forall i\in I$, $e_i^{T}(\tilde{\Sigma}_{II})^{-1}\tilde s_I$ either goes to $\infty$ or is a positive constant. Then as $\gamma\to\infty$, we have that 
	$$
	P(Y\geq \tilde s)\sim C\frac{\exp\{-(y^*)^{T}\tilde{\Sigma}^{-1}y^*/2\}}{\prod_{i\in I}e_i^{T}(\tilde{\Sigma}_{II})^{-1}\tilde s_I}
	$$
	where $C=C(\gamma)$ is a positive constant.
\label{thm:lemma}
\end{lemma}
\begin{proof}
	Given $x\in\mathbb{R}^d$, we define $\tilde{x}$ in the following way: $(\tilde{x})_i=(e_i^{T}(\tilde{\Sigma}_{II})^{-1}\tilde s_I)^{-1}x_i,\forall i\in I$; $(\tilde{x})_J=x_J$. Using $(3.4)$ in \cite{Hashorva2003gaussian}, we know that 
	$$
	(x+y^*)^{T}\tilde{\Sigma}^{-1}(x+y^*)=x^{T}\tilde{\Sigma}^{-1}x+2(x_I)^{T}(\tilde{\Sigma}_{II})^{-1}\tilde s_I+(y^*)^{T}\tilde{\Sigma}^{-1}y^*,
	$$
	and thus 
	$$
	\begin{aligned}
	\phi(\tilde{x}+y^*)&=(2\pi)^{-\frac{d}{2}}|\tilde{\Sigma}|^{-\frac12}\exp\{-\frac12\left[
	(\tilde{x})^{T}\tilde{\Sigma}^{-1}\tilde{x}+2(\tilde{x}_I)^{T}(\tilde{\Sigma}_{II})^{-1}\tilde s_I+(y^*)^{T}\tilde{\Sigma}^{-1}y^*\right]\}\\
	&=(2\pi)^{-\frac{d}{2}}|\tilde{\Sigma}|^{-\frac12}\exp\{-\frac12\left[
	(\tilde{x})^{T}\tilde{\Sigma}^{-1}\tilde{x}+2x_I^{T}\mathbf{1}_I+(y^*)^{T}\tilde{\Sigma}^{-1}y^*\right]\}
	\end{aligned}
	$$
	where $\phi$ is the density function of $N(\mathbf{0},\tilde{\Sigma})$. Then we get that 
	$$
	\begin{aligned}
	&P(Y\ge \tilde s)\\
	=&\int_{y\ge \tilde s}\phi(y)\mathrm{d}y\\
	=&\left(\prod_{i\in I}e_i^{T}(\tilde{\Sigma}_{II})^{-1}\tilde s_I\right)^{-1}\int_{x\ge \tilde{s}-y^*}\phi(\tilde{x}+y^*)\mathrm{d}x\\
	=&(2\pi)^{-\frac{d}{2}}|\tilde{\Sigma}|^{-\frac12}\left(\prod_{i\in I}e_i^{T}(\tilde{\Sigma}_{II})^{-1}\tilde s_I\right)^{-1}\exp\{-(y^*)^{T}\tilde{\Sigma}^{-1}y^*/2\}\int_{x\ge \tilde{s}-y^*}\exp\{-(\tilde{x})^{T}\tilde{\Sigma}^{-1}\tilde{x}/2-x_I^{T}\mathbf{1}_I\}\mathrm{d}x.
	\end{aligned}
	$$
	Apparent from the above, it suffices to show that $\int_{x\ge t-y^*}\exp\{-(\tilde{x})^{T}\tilde{\Sigma}^{-1}\tilde{x}/2-x_I^{T}\mathbf{1}_I\}\mathrm{d}x$ converges to a positive constant as $\gamma\to\infty$. Indeed, using $(3.6)$ in \cite{Hashorva2003gaussian} we know that $(\tilde{x})^{T}\tilde{\Sigma}^{-1}\tilde{x}\ge x_J^{T}(\tilde{\Sigma}_{JJ})^{-1}x_J$ and thus
	$$
	\exp\{-(\tilde{x})^{T}\tilde{\Sigma}^{-1}\tilde{x}/2-x_I^{T}\mathbf{1}_I\}\le \exp\{-x_J^{T}(\tilde{\Sigma}_{JJ})^{-1}x_J/2-x_I^{T}\mathbf{1}_I\}.
	$$
	Moreover, we have that 
	$$
	\int_{x_I\ge\mathbf{0}_I}\exp\{-x_J^{T}(\tilde{\Sigma}_{JJ})^{-1}x_J/2-x_I^{T}\mathbf{1}_I\}\mathrm{d}x=\int_{\mathbb{R}^{|J|}}\exp\{-x_J^{T}(\tilde{\Sigma}_{JJ})^{-1}x_J/2\}\mathrm{d}x_J<\infty.
	$$
	We partition $I$ into $I_1=\{i\in I:e_i^{T}(\tilde{\Sigma}_{II})^{-1}\tilde s_I\to\infty\}$ and $I_2=\{i\in I:e_i^{T}(\tilde{\Sigma}_{II})^{-1}\tilde s_I\text{ is a positive constant}\}$. Then we get that 
	$$
	\begin{aligned}
	&\lim_{\gamma\to\infty}\exp\{-(\tilde{x})^{T}\tilde{\Sigma}^{-1}\tilde{x}/2-x_I^{T}\mathbf{1}_I\}\\
	=&\exp\left\{-\frac12\left[(\tilde{x}_{I_2})^{T}(\tilde{\Sigma}^{-1})_{I_2I_2}\tilde{x}_{I_2}+(\tilde{x}_{I_2})^{T}(\tilde{\Sigma}^{-1})_{I_2J}x_J+x_J^{T}(\tilde{\Sigma}^{-1})_{JI_2}\tilde{x}_{I_2}+x_J^{T}(\tilde{\Sigma}^{-1})_{JJ}x_J\right]-x_I^{T}\mathbf{1}_I\right\}.
	\end{aligned}
	$$ 
	We know that the above limit does not depend on $\gamma$. By applying the dominated convergence theorem, we get that 
	$$
	\begin{aligned}
	&\lim_{\gamma\to\infty}\int_{x\ge \tilde{s}-y^*}\exp\{-(\tilde{x})^{T}\tilde{\Sigma}^{-1}\tilde{x}/2-x_I^{T}\mathbf{1}_I\}\mathrm{d}x\\
	=&\int\int_{x_I\ge\mathbf{0}_I,x_J\ge \tilde s_J^*}\\
	&\exp\left\{-\frac12\left[(\tilde{x}_{I_2})^{T}(\tilde{\Sigma}^{-1})_{I_2I_2}\tilde{x}_{I_2}+(\tilde{x}_{I_2})^{T}(\tilde{\Sigma}^{-1})_{I_2J}x_J+x_J^{T}(\tilde{\Sigma}^{-1})_{JI_2}\tilde{x}_{I_2}+x_J^{T}(\tilde{\Sigma}^{-1})_{JJ}x_J\right]-x_I^{T}\mathbf{1}_I\right\}\mathrm{d}x_I\mathrm{d}x_J\\
	=&\int\int_{x_{I_2}\ge\mathbf{0}_{I_2},x_J\ge \tilde s_J^*}\\
	&\exp\left\{-\frac12\left[(\tilde{x}_{I_2})^{T}(\tilde{\Sigma}^{-1})_{I_2I_2}\tilde{x}_{I_2}+(\tilde{x}_{I_2})^{T}(\tilde{\Sigma}^{-1})_{I_2J}x_J+x_J^{T}(\tilde{\Sigma}^{-1})_{JI_2}\tilde{x}_{I_2}+x_J^{T}(\tilde{\Sigma}^{-1})_{JJ}x_J\right]-x_{I_2}^{T}\mathbf{1}_{I_2}\right\}\mathrm{d}x_{I_2}\mathrm{d}x_J.
	\end{aligned}
	$$
	This shows that $\int_{x\ge t-y^*}\exp\{-(\tilde{x})^{T}\tilde{\Sigma}^{-1}\tilde{x}/2-x_I^{T}\mathbf{1}_I\}\mathrm{d}x$ converges to a positive constant as $\gamma\to\infty$, and hence we have proved the theorem.
\end{proof}
\par 

\begin{proof}[Proof of Theorem \ref{thm:efficiency}]
Suppose that $g(x)=g_i(x)$ for $h_{ij}(x)\geq 0,j=1,\dots,m_i$, $i=1,\dots,r'$ where $g_i$'s and $h_{ij}$'s are all affine functions. Then we can split $\{x:g(x)\geq \gamma\}$ into $\tilde{\mathcal{R}}_1,\dots,\tilde{\mathcal{R}}_{r'}$ where
$\tilde{\mathcal{R}}_i=\{x:g_i(x)\geq\gamma,h_{ij}(x)\geq 0,j=1,\dots,m_i\}$. 
We denote $\tilde{a}_i=\arg\min_x\{(x-\mu)^T\Sigma^{-1}(x-\mu):x\in\tilde{\mathcal{R}}_i\}$. 
\par 
To justify the asymptotic optimality of the proposed IS estimator \eqref{eq:proposed_IS}, we need to show that 
$$
\frac{\tilde{E}[Z^2]}{\tilde{E}[Z]^2}=\frac{E[I(g(X)\geq\gamma)L(X)]}{P(g(X)\geq\gamma)^2}=\frac{\sum_{i=1}^r E[I(X\in \mathcal{R}_i)L(X)]}{\left(\sum_{i=1}^{r'} P(X\in \tilde{\mathcal{R}}_i)\right)^2}
$$
is at most polynomially growing in $\gamma$. 
\par 
To simplify the notations, we consider the polyhedron $P_1:=\{x\in\R^d:Ax\geq t\}$ where $A\in \R^{m\times d},t\in \R^m$ and in particular, $t_1=\gamma+c$ for some constant $c\in\R$ and $t_2,\dots,t_m$ are all constants in $\R$. Naturally, we assume that $P(X\in P_1)>0$ where $X\sim N(\mu,\Sigma)$ for any $\gamma\in\R$. We define $x^*=\arg\min\{(x-\mu)^T\Sigma^{-1}(x-\mu):x\in P_1\}$. Note that for sufficiently large $\gamma$, each component of $x^*$ is an affine function of $\gamma$, so $(x^*-\mu)^T\Sigma^{-1}(x^*-\mu)$ is a quadratic polynomial of $\gamma$. We will prove that $-\log P(X\in P_1)\sim (x^*-\mu)^T\Sigma^{-1}(x^*-\mu)/2$ as $\gamma\rightarrow\infty$.
\par 
We use $A_i$ to denote the $i$-th row vector of $A$. Suppose that $A_{i_j}^{T}x\geq t_{i_j},j=1,\dots,m'$ are all the linearly independent active constraints at $x^*$. If $m'<d$, then we can add redundant constraints in the form of $x_{k_l}\ge -\infty, l=1,\cdots,d-m'$ such that we get $d$ linearly independent constraints now.
More specifically, let 
$$
B=\left(
\begin{matrix}
A_{i_1}^{T}\\
\vdots\\
A_{i_{m'}}^{T}\\
e_{k_1}^{T}\\
\vdots\\
e_{k_{d-m'}}^{T}
\end{matrix}\right),
s=\left(
\begin{matrix}
t_{i_1}\\
\vdots\\
t_{i_{m'}}\\
-\infty\\
\vdots\\
-\infty
\end{matrix}\right).
$$
By the definition, we get that $B$ is invertible. We know that for sufficiently large $\gamma$, the active constraints at $x^*$ do not change as $\gamma$ increases. Thus, in our following discussions, we assume that $B$ and $s$ does not change with $\gamma$. Also, it is clear that the constraint $A_1^{T}\ge t_1=\gamma+c$ must be active at $x^*$, i.e. $i_1=1$. Since $P_2:=\{x:Bx\ge s\}$ is obtained by removing constraints from $P_1$, we have that $P_1\subset P_2$. Our first step is to develop the asymptotic result of $P(X\in P_2)$, where we directly apply Lemma \ref{thm:lemma}.
\par 
We know that $Y:=B(X-\mu)\sim N(0,\tilde{\Sigma})$ where $\tilde \Sigma=B\Sigma B^{T}$ is positive definite. We denote $y^*=\arg\min\{y^{T}\tilde \Sigma^{-1}y:y\geq \tilde s\}$ where $\tilde s = s-B\mu$. 
Recall that under our settings, $s_1=\gamma+c$ for some constant $c\in\R$ so $\tilde s_1\rightarrow\infty$ as $\gamma\rightarrow\infty$. We still use the symbol $I$ to denote the set that satisfies \eqref{eqn:indexset}. Similar to our previous argument, $I$ does not change for sufficiently large $\gamma$. Also the limit $\lim_{\gamma\to\infty}(\tilde s-y^*)_J$ exists. For any $i\in I$, we know that $e_i^{T}(\tilde \Sigma_{II})^{-1}\tilde s_I>0$ and it is an affine function of $\gamma$, and thus either it goes to $\infty$ or it is a positive constant as $\gamma\rightarrow\infty$. In conclusion, all the assumptions of Lemma \ref{thm:lemma} hold in this case. Therefore, we get that 
\begin{equation}
P(X\in P_2)\sim C\frac{\exp\{-(y^*)^{T}\tilde \Sigma^{-1}y^*/2\}}{\prod_{i\in I}e_i^{T}(\tilde \Sigma_{II})^{-1}\tilde s_I}
\label{eqn:asyptoticP2}
\end{equation}
for some constant $C$. It is easy to verify that $(y^*)^{T}\tilde \Sigma^{-1}y^*=(x^*-\mu)^T\Sigma^{-1}(x^*-\mu)$, and hence $-\log P(X\in P_2)\sim (x^*-\mu)^T\Sigma^{-1}(x^*-\mu)/2$.
\par 

Clearly $P(X\in P_2)$ gives an upper bound for $P(X\in P_1)$. Now we develop a lower bound using similar techniques. \par 

We denote $x^{**}=\arg\min_x\{(x-\mu)^T\Sigma^{-1}(x-\mu):x\in\overline{P_2\setminus P_1}\}$ and $x^{***}=\arg\min_x\{(x-x^*)^T\Sigma^{-1}(x-x^*):x\in\overline{P_2\setminus P_1}\}$. Clearly each component of $x^*$ and $x^{***}$ is affine in $\gamma$ when $\gamma$ is sufficiently large, and hence $(x^{***}-x^*)^T\Sigma^{-1}(x^{***}-x^*)\geq 0$ is polynomial in $\gamma$. Thus we know that $(x^{***}-x^*)^T\Sigma^{-1}(x^{***}-x^*)$ either goes to infinity or stays a nonnegative constant as $\gamma\rightarrow\infty$. However, if $(x^{***}-x^*)^T\Sigma^{-1}(x^{***}-x^*)=0$ for sufficiently large $\gamma$, then we have that $x^{***}=x^*$, and hence $x^*\in \overline{P_2\setminus P_1}$, which contradicts the easily verified fact that $(x^{**}-\mu)^T\Sigma^{-1}(x^{**}-\mu)>(x^*-\mu)^T\Sigma^{-1}(x^*-\mu)$. Therefore, there exists a constant $0<\varepsilon<1$ such that $\{x:(x-x^*)^T\Sigma^{-1}(x-x^*)\leq \varepsilon^2\}\cap P_1=\{x:(x-x^*)^T\Sigma^{-1}(x-x^*)\leq \varepsilon^2\}\cap P_2$ for sufficiently large $\gamma$. Correspondingly, there exists $\varepsilon'>0$ such that $\{x:\Vert x\Vert_{\infty}\leq\varepsilon'\}\subseteq\{x:x^T\Sigma^{-1}x\leq \varepsilon^2\}$. 

Still we define $Y=B(X-\mu)\sim N(0,\tilde{\Sigma})$. Then we get that 
$$
\begin{aligned}
P(X\in P_1)&\geq P((X-x^*)^T\Sigma^{-1}(X-x^*)\leq\varepsilon^2,X\in P_1)\\
&=P((X-x^*)^T\Sigma^{-1}(X-x^*)\leq\varepsilon^2,X\in P_2)\\
&=P((Y+B\mu-Bx^*)^T\tilde{\Sigma}^{-1}(Y+B\mu-Bx^*)\leq\varepsilon^2,Y\geq \tilde{s}).
\end{aligned}
$$
Similar to the proof of Lemma \ref{thm:lemma}, we have that 
$$
\begin{aligned}
&P(X\in P_1)\\
\geq& \int_{(y+B\mu-Bx^*)^T\tilde{\Sigma}^{-1}(y+B\mu-Bx^*)\leq\varepsilon^2,y\geq\tilde{s}}\phi(y)\mathrm{d}y\\
=&\left(\prod_{i\in I}e_i^T(\tilde{\Sigma}_{II})^{-1}\tilde{s}_I\right)^{-1}\int_{\tilde{x}^T\tilde{\Sigma}^{-1}\tilde{x}\leq\varepsilon^2,\tilde{x}\geq \tilde{s}-y^*}\phi(\tilde{x}+y^*)\mathrm{d}x\\
\geq & (2\pi)^{-\frac{d}{2}}|\tilde{\Sigma}|^{-\frac12}\left(\prod_{i\in I}e_i^{T}(\tilde{\Sigma}_{II})^{-1}\tilde s_I\right)^{-1}\exp\{-(y^*)^{T}\tilde{\Sigma}^{-1}y^*/2\}(1-\varepsilon^2/2)\int_{0\leq \tilde{x}\leq \varepsilon'\mathbf{1}}\exp\{-x_I^{T}\mathbf{1}_I\}\mathrm{d}x\\
=&(2\pi)^{-\frac{d}{2}}|\tilde{\Sigma}|^{-\frac12}(1-\varepsilon^2/2)\varepsilon'^{|J|}\left(\prod_{i\in I}\frac{1-\exp\{-e_i^{T}(\tilde{\Sigma}_{II})^{-1}\tilde s_I\varepsilon'\}}{e_i^{T}(\tilde{\Sigma}_{II})^{-1}\tilde s_I}\right)\exp\{-(y^*)^{T}\tilde{\Sigma}^{-1}y^*/2\}.
\end{aligned}
$$

Combining the upper and lower bound for $P(X\in P_1)$, we finally get that $-\log P(X\in P_1)\sim (x^*-\mu)^T\Sigma^{-1}(x^*-\mu)/2$ as $\gamma\rightarrow\infty$. We apply this result to $\tilde{\mathcal{R}}_i,i=1,\dots,s$ to get that $-\log P(X\in \tilde{\mathcal{R}}_i)\sim(\tilde a_i-\mu)^T\Sigma^{-1}(\tilde a_i-\mu)/2$, which implies that 
\begin{equation}
-\log P(g(X)\geq \gamma)=-\log\left(\sum_{i=1}^s P(X\in\tilde{\mathcal{R}}_i)\right)\sim \min_{i=1,\dots,s}\{(\tilde a_i-\mu)^T\Sigma^{-1}(\tilde a_i-\mu)\}/2=(a_1-\mu)^T\Sigma^{-1}(a_1-\mu)/2.
\label{eqn:corollary}
\end{equation}
\par 
Moreover, since 
$$
L(x)\leq \frac{re^{-(x-\mu)^T\Sigma^{-1}(x-\mu)/2}}{e^{-(x-a_i)^{T}\Sigma^{-1}(x-a_i)/2}}=re^{-(a_i-\mu)^T\Sigma^{-1}(a_i-\mu)/2-(a_i-\mu)^{T}\Sigma^{-1}(x-a_i)}
$$
and $(a_i-\mu)^{T}\Sigma^{-1}(x-a_i)\geq 0$ on $\mathcal{R}_i$, we get that 
$$
E[I(X\in\mathcal{R}_i)L(X)]\leq re^{-(a_i-\mu)^T\Sigma^{-1}(a_i-\mu)/2}P(X\in \mathcal{R}_i)\leq re^{-(a_1-\mu)^T\Sigma^{-1}(a_1-\mu)/2}P(X\in\mathcal{R}_i),
$$
and hence $\tilde{E}[Z^2]\leq re^{-(a_1-\mu)^T\Sigma^{-1}(a_1-\mu)/2}P(g(X)\geq\gamma)$. Combining the inequality with the asymptotic result for $P(g(X)\geq\gamma)$, we can easily get that the IS estimator $Z$ is asymptotically optimal.
\end{proof}

\begin{proof}[Proof of Corollary \ref{thm:asymptotic}]
See \eqref{eqn:corollary} in the proof of Theorem \ref{thm:efficiency}.
\end{proof}

\begin{proof}[Proof of Corollary \ref{thm:corollary}]
Now we suppose that $X\sim \sum_{j=1}^m \pi_j N(\mu_j,\Sigma_j)$. We know that 
$$
P(g(X)\geq\gamma)=\sum_{j=1}^m \pi_j P(g(X)\geq\gamma|X\sim N(\mu_j,\Sigma_j))
$$
and thus 
$$
-\log P(g(X)\geq \gamma)\sim \min_{j=1,\dots,m}\{(a_{j1}-\mu_j)^T\Sigma_j^{-1}(a_{j1}-\mu_j)/2\}.
$$
Moreover, we have that 
$$
\frac{e^{-(x-\mu_j)^T\Sigma_j^{-1}(x-\mu_j)/2}}{\sum_{i=1}^{r_j}1/r_j e^{-(x-a_{ji})^T\Sigma_j^{-1}(x-a_{ji})/2}}\leq r_j e^{-(a_{j1}-\mu_j)^T\Sigma_j^{-1}(a_{j1}-\mu_j)/2}\leq \max_j\{r_j\}e^{-\min_{j}\{(a_{j1}-\mu_j)^T\Sigma_j^{-1}(a_{j1}-\mu_j)/2\}}
$$
and hence 
$$
L(x)\leq \max_j\{r_j\}e^{-\min_{j}\{(a_{j1}-\mu_j)^T\Sigma_j^{-1}(a_{j1}-\mu_j)/2\}}.
$$
Therefore, we get that
$$
\frac{\tilde{E}[Z^2]}{\tilde{E}[Z]^2}=\frac{E[I(g(X)\geq\gamma)L(X)]}{P(g(X)\geq\gamma)^2}\leq \frac{\max_j\{r_j\}e^{-\min_{j}\{(a_{j1}-\mu_j)^T\Sigma_j^{-1}(a_{j1}-\mu_j)/2\}}}{P(g(X)\geq\gamma)}
$$
grows polynomially in $\gamma$ and the IS estimator $Z$ is asymptotically optimal.
\end{proof}

\section*{Acknowledgements}
We gratefully acknowledge support from the National Science Foundation under grants CAREER CMMI-1653339/1834710, IIS-1849280, IIS-1849304, and the Manufacturing Futures Initiative at Carnegie Mellon University. A preliminary conference version of this work has appeared in \cite{huang2018designing}.

\bibliographystyle{plain}
\bibliography{citation.bib}

\begin{thebibliography}{100}

\bibitem{ABJ06}
T.~P.~I. Ahamed, V.~S. Borkar, and S.~Juneja.
\newblock Adaptive importance sampling technique for {M}arkov chains using
  stochastic approximation.
\newblock {\em Operations Research}, 54:489--504, 2006.

\bibitem{ahn2018efficient}
Dohyun Ahn and Kyoung-Kuk Kim.
\newblock Efficient simulation for expectations over the union of half-spaces.
\newblock {\em ACM Transactions on Modeling and Computer Simulation (TOMACS)},
  28(3):1--20, 2018.

\bibitem{pASM85a}
S.~Asmussen.
\newblock Conjugate processes and the simulation of ruin problems.
\newblock {\em Stochastic Processes and their Applications}, 20:213--229, 1985.

\bibitem{AsmBin97}
S.~Asmussen and K.~Binswanger.
\newblock Simulation of ruin probabilities for subexponential claims.
\newblock {\em Astin Bulletin}, 27:297--318, 1997.

\bibitem{AsmKro06}
S.~Asmussen and D.~Kroese.
\newblock Improved algorithms for rare event simulation with heavy tails.
\newblock {\em Advances in Applied Probability}, 38:545--558, 2006.

\bibitem{ASM00Ruin}
S{\o}ren Asmussen and Hansj{\"o}rg Albrecher.
\newblock {\em Ruin probabilities}, volume~14.
\newblock World scientific, 2010.

\bibitem{asmussen2007stochastic}
S{\o}ren Asmussen and Peter~W Glynn.
\newblock {\em Stochastic Simulation: Algorithms and Analysis}, volume~57.
\newblock Springer Science \& Business Media, New York, 2007.

\bibitem{asuncion2007uci}
Arthur Asuncion and David Newman.
\newblock Uci machine learning repository, 2007.

\bibitem{BKS07}
M.~Bayati, J.~Kim, and A.~Saberi.
\newblock A sequential algorithm for generating random graphs.
\newblock {\em Approximation, Randomization and combinatorial Optimization.
  Algorithms and Techniques. Lecture Notes in Computer Science}, 4627:326--340,
  2007.

\bibitem{Blanchet09AAP}
J.~Blanchet.
\newblock Efficient importance sampling for binary contingency tables.
\newblock {\em Annals of Applied Probability}, 19:949--982, 2009.

\bibitem{BGL10}
J.~Blanchet, P.~Glynn, and H.~Lam.
\newblock Rare event simulation for a slotted time {$M/G/s$} model.
\newblock {\em Queueing Systems}, 63:33--57, 2009.

\bibitem{BMRT09}
J.~Blanchet and M.~Mandjes.
\newblock Rare event simulation for queues.
\newblock In {\em Rare Event Simulation Using Monte Carlo Methods}, pages
  87--124. 2009.
\newblock Chapter 5.

\bibitem{blanchet2008efficient}
Jose Blanchet and Peter Glynn.
\newblock Efficient rare-event simulation for the maximum of heavy-tailed
  random walks.
\newblock {\em The Annals of Applied Probability}, pages 1351--1378, 2008.

\bibitem{blanchet2012state}
Jose Blanchet and Henry Lam.
\newblock State-dependent importance sampling for rare-event simulation: An
  overview and recent advances.
\newblock {\em Surveys in Operations Research and Management Science},
  17(1):38--59, 2012.

\bibitem{blanchet2014rare}
Jose Blanchet and Henry Lam.
\newblock Rare-event simulation for many-server queues.
\newblock {\em Mathematics of Operations Research}, 39(4):1142--1178, 2014.

\bibitem{blanchet2012efficient}
Jose Blanchet, Henry Lam, and Bert Zwart.
\newblock Efficient rare-event simulation for perpetuities.
\newblock {\em Stochastic Processes and their Applications},
  122(10):3361--3392, 2012.

\bibitem{blanchet2008state}
Jose~H Blanchet and Jingchen Liu.
\newblock State-dependent importance sampling for regularly varying random
  walks.
\newblock {\em Advances in Applied Probability}, 40(4):1104--1128, 2008.

\bibitem{board2018preliminary}
National Transpotation~Safety Board.
\newblock Preliminary report, highway, hwy18mh010, 2018.

\bibitem{board2019report}
National Transpotation~Safety Board.
\newblock Collision between car operating with partial driving automation and
  truck-tractor semitrailer delray beach, florida, march 1, 2019, 2019.

\bibitem{bock2004methods}
RK~Bock, A~Chilingarian, M~Gaug, F~Hakl, Th~Hengstebeck, M~Ji{\v{r}}ina,
  J~Klaschka, E~Kotr{\v{c}}, P~Savick{\`y}, S~Towers, et~al.
\newblock Methods for multidimensional event classification: a case study using
  images from a cherenkov gamma-ray telescope.
\newblock {\em Nuclear Instruments and Methods in Physics Research Section A:
  Accelerators, Spectrometers, Detectors and Associated Equipment},
  516(2-3):511--528, 2004.

\bibitem{DNR00}
P.~T.~De Boer, V.F. Nicola, and R.Y. Rubinstein.
\newblock Adaptive importance sampling simulation of queueing networks.
\newblock In {\em Proceedings of 2000 Winter Simulation Conference}, pages
  646--655. IEEE Press, 2000.

\bibitem{BJK04}
V.~S. Borkar, S.~Juneja, and A.~A. Kherani.
\newblock Performance analysis conditioned on rare events: An adaptive
  simulation scheme.
\newblock {\em Communications in Information}, 3:259--278, 2004.

\bibitem{botev2020sampling}
Zdravko~I Botev and Pierre L’Ecuyer.
\newblock Sampling conditionally on a rare event via generalized splitting.
\newblock {\em INFORMS Journal on Computing}, 2020.

\bibitem{botev2013markov}
Zdravko~I Botev, Pierre L’Ecuyer, and Bruno Tuffin.
\newblock Markov chain importance sampling with applications to rare event
  probability estimation.
\newblock {\em Statistics and Computing}, 23(2):271--285, 2013.

\bibitem{bucklew2013introduction}
James Bucklew.
\newblock {\em Introduction to Rare Event Simulation}.
\newblock Springer Science \& Business Media, New York, 2013.

\bibitem{chan2012improved}
Joshua~CC Chan and Dirk~P Kroese.
\newblock Improved cross-entropy method for estimation.
\newblock {\em Statistics and Computing}, 22(5):1031--1040, 2012.

\bibitem{chen2019efficient}
Bohan Chen, Jose Blanchet, Chang-Han Rhee, and Bert Zwart.
\newblock Efficient rare-event simulation for multiple jump events in regularly
  varying random walks and compound poisson processes.
\newblock {\em Mathematics of Operations Research}, 44(3):919--942, 2019.

\bibitem{chen2015deepdriving}
Chenyi Chen, Ari Seff, Alain Kornhauser, and Jianxiong Xiao.
\newblock Deepdriving: Learning affordance for direct perception in autonomous
  driving.
\newblock In {\em Proceedings of the IEEE International Conference on Computer
  Vision}, pages 2722--2730, 2015.

\bibitem{chen2017end}
Zhilu Chen and Xinming Huang.
\newblock End-to-end learning for lane keeping of self-driving cars.
\newblock In {\em 2017 IEEE Intelligent Vehicles Symposium (IV)}, pages
  1856--1860. IEEE, 2017.

\bibitem{chua1988cellular}
Leon~O Chua and Lin Yang.
\newblock Cellular neural networks: Theory.
\newblock {\em IEEE Transactions on circuits and systems}, 35(10):1257--1272,
  1988.

\bibitem{collamore2002importance}
Jeffrey~F Collamore.
\newblock Importance sampling techniques for the multidimensional ruin problem
  for general markov additive sequences of random vectors.
\newblock {\em The Annals of Applied Probability}, 12(1):382--421, 2002.

\bibitem{DEBKROMANRUB05}
P.~T. de~Boer, D.~Kroese, S.~Mannor, and R.~Rubinstein.
\newblock A tutorial on the cross-entropy method.
\newblock {\em Annals of Operations Research}, 134:19--67, 2005.

\bibitem{Dean2009562}
Thomas Dean and Paul Dupuis.
\newblock Splitting for rare event simulation: A large deviation approach to
  design and analysis.
\newblock {\em Stochastic Processes and their Applications}, 119(2):562 -- 587,
  2009.

\bibitem{DG01}
P.~Y. Desai and P.~W. Glynn.
\newblock A {M}arkov chain perspective on adaptive {M}onte {C}arlo algorithms.
\newblock {\em Proceedings of 2001 Winter Simulation Conference}, 9:391--412,
  2001.

\bibitem{dieker2005asymptotically}
AB~Dieker and Michel Mandjes.
\newblock On asymptotically efficient simulation of large deviation
  probabilities.
\newblock {\em Advances in applied probability}, 37(2):539--552, 2005.

\bibitem{dreossi2019compositional}
Tommaso Dreossi, Alexandre Donz{\'e}, and Sanjit~A Seshia.
\newblock Compositional falsification of cyber-physical systems with machine
  learning components.
\newblock {\em Journal of Automated Reasoning}, 63(4):1031--1053, 2019.

\bibitem{dupuis2007importance}
Paul Dupuis, Kevin Leder, and Hui Wang.
\newblock Importance sampling for sums of random variables with regularly
  varying tails.
\newblock {\em ACM Transactions on Modeling and Computer Simulation (TOMACS)},
  17(3):14--es, 2007.

\bibitem{dupuis2009importance}
Paul Dupuis, Kevin Leder, and Hui Wang.
\newblock Importance sampling for weighted-serve-the-longest-queue.
\newblock {\em Mathematics of Operations Research}, 34(3):642--660, 2009.

\bibitem{dupuis2012importance}
Paul Dupuis, Konstantinos Spiliopoulos, and Hui Wang.
\newblock Importance sampling for multiscale diffusions.
\newblock {\em Multiscale Modeling \& Simulation}, 10(1):1--27, 2012.

\bibitem{dvovrak2007softening}
Jakub Dvo{\v{r}}{\'a}k and Petr Savick{\`y}.
\newblock Softening splits in decision trees using simulated annealing.
\newblock In {\em International Conference on Adaptive and Natural Computing
  Algorithms}, pages 721--729. Springer, 2007.

\bibitem{fraade2018measuring}
Laura Fraade-Blanar, Marjory~S Blumenthal, James~M Anderson, and Nidhi Kalra.
\newblock {\em Measuring automated vehicle safety: forging a framework}.
\newblock RAND Corporation, 2018.

\bibitem{glasius1995neural}
Roy Glasius, Andrzej Komoda, and Stan~CAM Gielen.
\newblock Neural network dynamics for path planning and obstacle avoidance.
\newblock {\em Neural Networks}, 8(1):125--133, 1995.

\bibitem{Glasserman04}
P.~Glasserman.
\newblock {\em Monte Carlo Methods in Financial Engineering}.
\newblock Springer, 2004.

\bibitem{GHSZ99}
P.~Glasserman, P.~Heidelberger, P.~Shahabuddin, and T.~Zajic.
\newblock Multilevel splitting for estimating rare event probabilities.
\newblock {\em Operations Research}, 47:585--600, 1999.

\bibitem{glasserman2013monte}
Paul Glasserman.
\newblock {\em Monte Carlo Methods in Financial Engineering}, volume~53.
\newblock Springer Science \& Business Media, New York, 2013.

\bibitem{glasserman1998large}
Paul Glasserman, Philip Heidelberger, Perwez Shahabuddin, and Tim Zajic.
\newblock A large deviations perspective on the efficiency of multilevel
  splitting.
\newblock {\em IEEE Transactions on Automatic Control}, 43(12):1666--1679,
  1998.

\bibitem{glasserman2008fast}
Paul Glasserman, Wanmo Kang, and Perwez Shahabuddin.
\newblock Fast simulation of multifactor portfolio credit risk.
\newblock {\em Operations Research}, 56(5):1200--1217, 2008.

\bibitem{glasserman2005importance}
Paul Glasserman and Jingyi Li.
\newblock Importance sampling for portfolio credit risk.
\newblock {\em Management Science}, 51(11):1643--1656, 2005.

\bibitem{glasserman1997counterexamples}
Paul Glasserman and Yashan Wang.
\newblock Counterexamples in importance sampling for large deviations
  probabilities.
\newblock {\em The Annals of Applied Probability}, 7(3):731--746, 1997.

\bibitem{glynn1989importance}
Peter~W Glynn and Donald~L Iglehart.
\newblock Importance sampling for stochastic simulations.
\newblock {\em Management Science}, 35(11):1367--1392, 1989.

\bibitem{goodfellow2016deep}
Ian Goodfellow, Yoshua Bengio, Aaron Courville, and Yoshua Bengio.
\newblock {\em Deep Learning}, volume~1.
\newblock MIT press Cambridge, Massachusetts, 2016.

\bibitem{goodfellow2014explaining}
Ian~J Goodfellow, Jonathon Shlens, and Christian Szegedy.
\newblock Explaining and harnessing adversarial examples.
\newblock {\em arXiv preprint arXiv:1412.6572}, 2014.

\bibitem{grace2014automated}
Adam~W Grace, Dirk~P Kroese, and Werner Sandmann.
\newblock Automated state-dependent importance sampling for markov jump
  processes via sampling from the zero-variance distribution.
\newblock {\em Journal of Applied Probability}, 51(3):741--755, 2014.

\bibitem{Grassberger02}
P.~Grassberger.
\newblock Go with the winners: {A} general {M}onte {C}arlo strategy.
\newblock {\em Computer Physics Communications}, 147:64--70, 2002.

\bibitem{Hashorva2003gaussian}
Enkelejd Hashorva and Juerg Huesler.
\newblock On multivariate gaussian tails.
\newblock {\em Annals of the Institute of Statistical Mathematics},
  55:507--522, 02 2003.

\bibitem{Heidelberger95}
P.~Heidelberger.
\newblock Fast simulation of rare events in queueing and reliability models.
\newblock {\em ACM Transactions on Modeling and Computer Simulation (TOMACS)},
  5:43--85, 1995.

\bibitem{honnappa2018dominating}
Harsha Honnappa, Raghu Pasupathy, and Prateek Jaiswal.
\newblock Dominating points of gaussian extremes, 2018.

\bibitem{huang2017accelerated}
Z.~Huang, H.~Lam, D.~J. LeBlanc, and D.~Zhao.
\newblock Accelerated evaluation of automated vehicles using piecewise mixture
  models.
\newblock {\em IEEE Transactions on Intelligent Transportation Systems}, pages
  1--11, 2017.

\bibitem{huang2018designing}
Zhiyuan Huang, Henry Lam, and Ding Zhao.
\newblock Designing importance samplers to simulate machine learning predictors
  via optimization.
\newblock In {\em 2018 Winter Simulation Conference (WSC)}, pages 1730--1741.
  IEEE, 2018.

\bibitem{huang2018rare}
Zhiyuan Huang, Henry Lam, and Ding Zhao.
\newblock Rare-event simulation without structural information: a
  learning-based approach.
\newblock In {\em 2018 Winter Simulation Conference (WSC)}, pages 1826--1837.
  IEEE, 2018.

\bibitem{hult2012importance}
Henrik Hult and Jens Svensson.
\newblock On importance sampling with mixtures for random walks with heavy
  tails.
\newblock {\em ACM Transactions on Modeling and Computer Simulation (TOMACS)},
  22(2):1--21, 2012.

\bibitem{juneja2006criteria}
S.~Juneja and P.~Shahabuddin.
\newblock Chapter 11 rare-event simulation techniques: An introduction and
  recent advances.
\newblock In Shane~G. Henderson and Barry~L. Nelson, editors, {\em Simulation},
  volume~13 of {\em Handbooks in Operations Research and Management Science},
  pages 291 -- 350. Elsevier, 2006.

\bibitem{juneja2006rare}
Sandeep Juneja and Perwez Shahabuddin.
\newblock Rare-event simulation techniques: An introduction and recent
  advances.
\newblock {\em Handbooks in Operations Research and Management Science},
  13:291--350, 2006.

\bibitem{kalra2016driving}
Nidhi Kalra and Susan~M Paddock.
\newblock Driving to safety: How many miles of driving would it take to
  demonstrate autonomous vehicle reliability?
\newblock {\em Transportation Research Part A: Policy and Practice},
  94:182--193, 2016.

\bibitem{KESWALCHA93}
G.~Kesidis, J.~Walrand, and C.-S. Chang.
\newblock Effective bandwidths for multiclass {M}arkov fluids and other {ATM}
  sources.
\newblock {\em IEEE/ACM Transactions on Networks.}, 1:424--428, 1993.

\bibitem{kochenderfer2012next}
Mykel~J Kochenderfer, Jessica~E Holland, and James~P Chryssanthacopoulos.
\newblock Next-generation airborne collision avoidance system.
\newblock Technical report, Massachusetts Institute of Technology-Lincoln
  Laboratory Lexington United States, 2012.

\bibitem{KBCP99}
C.~Kollman, K.~Baggerly, D.~Cox, and R.~Picard.
\newblock Adaptive importance sampling on discrete {M}arkov chains.
\newblock {\em Annals of Applied Probability}, 9:391--412, 1999.

\bibitem{koopman2017autonomous}
Philip Koopman and Michael Wagner.
\newblock Autonomous vehicle safety: An interdisciplinary challenge.
\newblock {\em IEEE Intelligent Transportation Systems Magazine}, 9(1):90--96,
  2017.

\bibitem{KN99}
D.~P. Kroese and V.~F. Nicola.
\newblock Efficient estimation of overflow probabilities in queues with
  breakdowns.
\newblock {\em Performance Evaluation}, 36-37:471--484, 1999.

\bibitem{kurakin2016adversarial1}
Alexey Kurakin, Ian Goodfellow, and Samy Bengio.
\newblock Adversarial examples in the physical world.
\newblock {\em arXiv preprint arXiv:1607.02533}, 2016.

\bibitem{kurakin2016adversarial}
Alexey Kurakin, Ian Goodfellow, and Samy Bengio.
\newblock Adversarial machine learning at scale.
\newblock {\em arXiv preprint arXiv:1611.01236}, 2016.

\bibitem{LLLT09}
P.~L'Ecuyer, F.~Le Gland, P.~Lezaud, and B.~Tuffin.
\newblock Splitting techniques.
\newblock In {\em Rare Event Simulation Using Monte Carlo Methods}, pages
  39--62. 2009.
\newblock Chapter 3.

\bibitem{ecuyer2010criteria}
Pierre L'Ecuyer, Jose~H. Blanchet, Bruno Tuffin, and Peter~W. Glynn.
\newblock Asymptotic robustness of estimators in rare-event simulation.
\newblock {\em ACM Trans. Model. Comput. Simul.}, 20(1), February 2010.

\bibitem{mivsic2017optimization}
Velibor~V Mi{\v{s}}ic.
\newblock Optimization of tree ensembles.
\newblock {\em Working Paper: arXiv preprint arXiv:1705.10883}, 2017.

\bibitem{muller2006off}
Urs Muller, Jan Ben, Eric Cosatto, Beat Flepp, and Yann~L Cun.
\newblock Off-road obstacle avoidance through end-to-end learning.
\newblock In {\em Advances in neural information processing systems}, pages
  739--746, 2006.

\bibitem{murthy2015state}
Karthyek~RA Murthy, Sandeep Juneja, and Jose Blanchet.
\newblock State-independent importance sampling for random walks with regularly
  varying increments.
\newblock {\em Stochastic Systems}, 4(2):321--374, 2015.

\bibitem{nicola1993fast}
Victor~F Nicola, Marvin~K Nakayama, Philip Heidelberger, and Ambuj Goyal.
\newblock Fast simulation of highly dependable systems with general failure and
  repair processes.
\newblock {\em IEEE Transactions on Computers}, 42(12):1440--1452, 1993.

\bibitem{nicola2001techniques}
Victor~F Nicola, Perwez Shahabuddin, and Marvin~K Nakayama.
\newblock Techniques for fast simulation of models of highly dependable
  systems.
\newblock {\em IEEE Transactions on Reliability}, 50(3):246--264, 2001.

\bibitem{o2018scalable}
Matthew O'Kelly, Aman Sinha, Hongseok Namkoong, Russ Tedrake, and John~C Duchi.
\newblock Scalable end-to-end autonomous vehicle testing via rare-event
  simulation.
\newblock In {\em Advances in Neural Information Processing Systems}, pages
  9827--9838, 2018.

\bibitem{owen2019importance}
Art~B Owen, Yury Maximov, Michael Chertkov, et~al.
\newblock Importance sampling the union of rare events with an application to
  power systems analysis.
\newblock {\em Electronic Journal of Statistics}, 13(1):231--254, 2019.

\bibitem{GLAHEISHAZAJ98}
{P. Glasserman, P. Heidelberger and P. Shahabuddin and T. Zajic}.
\newblock A large deviations perspective on the effiency of multilevel
  splitting.
\newblock {\em IEEE Transactions on Automated Control}, pages 1666--1679, 1998.

\bibitem{PARWAL89}
S.~Parekh and J.~Walrand.
\newblock Quick simulation of rare events in networks.
\newblock {\em IEEE Transactions on Automatic Control}, 34:54--66, 1989.

\bibitem{Rubinstein99}
{R. Y. Rubinstein}.
\newblock Rare-event simulation via cross-entropy and importance sampling.
\newblock {\em Second Workshop on Rare Event Simulation, RESIM’99}, pages
  1--17, 1999.

\bibitem{Ridder09}
A.~Ridder.
\newblock Importance sampling algorithms for first passage time probabilities
  in the infinite server queue.
\newblock {\em European Journal of Operational Research}, 199:176--186, 2009.

\bibitem{RT09}
G.~Rubino and B.~Tuffin.
\newblock Markovian models for dependability analysis.
\newblock In {\em Rare Event Simulation Using Monte Carlo Methods}, pages
  125--144. 2009.
\newblock Chapter 6.

\bibitem{RUBKRO04}
R.~Rubinstein and D.~Kroese.
\newblock {\em The Cross-Entropy Method: A Unified Approach to Combinatorial
  Optimization, Monte-Carlo Simulation, and Machine Learning}.
\newblock Springer-Verlag, 2004.

\bibitem{Rubinstein97}
R.~Y. Rubinstein.
\newblock Optimization of computer simulation models with rare events.
\newblock {\em European Journal of Operations Research}, 99:89--112, 1997.

\bibitem{rubinstein2016simulation}
Reuven~Y Rubinstein and Dirk~P Kroese.
\newblock {\em Simulation and the Monte Carlo Method}, volume~10.
\newblock John Wiley \& Sons, New Jersey, 2016.

\bibitem{Sadowsky91}
J.~S. Sadowsky.
\newblock Large deviations theory and efficient simulation of excessive
  backlogs in a {$GI/GI/m$} queue.
\newblock {\em IEEE Transactions on Automatic Control}, 36:1383--1394, 1991.

\bibitem{sadowsky1990large}
John~S Sadowsky and James~A Bucklew.
\newblock On large deviations theory and asymptotically efficient monte carlo
  estimation.
\newblock {\em IEEE transactions on Information Theory}, 36(3):579--588, 1990.

\bibitem{Sandmann09}
W.~Sandmann.
\newblock Rare event simulation methodologies in systems biology.
\newblock In {\em Rare Event Simulation Using Monte Carlo Methods}, pages
  243--266. 2009.
\newblock Chapter 11.

\bibitem{savicky2004experimental}
Petr Savick{\`y} and Emil Kotrc.
\newblock Experimental study of leaf confidences for random forest.
\newblock In {\em Proceedings of the 16th Symposium on Computational
  Statistics}, pages 1767--1774. Prague, Czech Republic, 2004.

\bibitem{siegmund1976importance}
David Siegmund.
\newblock Importance sampling in the monte carlo study of sequential tests.
\newblock {\em The Annals of Statistics}, pages 673--684, 1976.

\bibitem{spielberg2019neural}
Nathan~A Spielberg, Matthew Brown, Nitin~R Kapania, John~C Kegelman, and
  J~Christian Gerdes.
\newblock Neural network vehicle models for high-performance automated driving.
\newblock {\em Science Robotics}, 4(28), 2019.

\bibitem{SP02}
R.~Szechtman and P.~Glynn.
\newblock Rare event simulation for infinite server queues.
\newblock In {\em Proceedings of the 2002 Winter Simulation Conference}, pages
  416--423, 2002.

\bibitem{tjeng2017verifying}
Vincent Tjeng and Russ Tedrake.
\newblock Verifying neural networks with mixed integer programming.
\newblock {\em Working Paper: arXiv preprint arXiv:1711.07356}, 2017.

\bibitem{Tuffin04QEST}
B.~Tuffin.
\newblock On numerical problems in simulation of highly reliable {M}arkovian
  systems.
\newblock In {\em Proceedings of the 1st International Conference on
  Quantitative Evaluation of SysTems (QEST)}, pages 156--164. IEEE Computer
  Society Press, 2004.

\bibitem{uesato2018rigorous}
Jonathan Uesato, Ananya Kumar, Csaba Szepesvari, Tom Erez, Avraham Ruderman,
  Keith Anderson, Nicolas Heess, Pushmeet Kohli, et~al.
\newblock Rigorous agent evaluation: An adversarial approach to uncover
  catastrophic failures.
\newblock {\em arXiv preprint arXiv:1812.01647}, 2018.

\bibitem{van2018autonomous}
Jessica Van~Brummelen, Marie O’Brien, Dominique Gruyer, and Homayoun
  Najjaran.
\newblock Autonomous vehicle perception: The technology of today and tomorrow.
\newblock {\em Transportation research part C: emerging technologies},
  89:384--406, 2018.

\bibitem{vanden2012rare}
Eric Vanden-Eijnden and Jonathan Weare.
\newblock Rare event simulation of small noise diffusions.
\newblock {\em Communications on Pure and Applied Mathematics},
  65(12):1770--1803, 2012.

\bibitem{wang2019statistically}
Benjie Wang, Stefan Webb, and Tom Rainforth.
\newblock Statistically robust neural network classification.
\newblock {\em arXiv preprint arXiv:1912.04884}, 2019.

\bibitem{webb2018statistical}
Stefan Webb, Tom Rainforth, Yee~Whye Teh, and M~Pawan Kumar.
\newblock A statistical approach to assessing neural network robustness.
\newblock {\em arXiv preprint arXiv:1811.07209}, 2018.

\bibitem{weng2018proven}
Tsui-Wei Weng, Pin-Yu Chen, Lam~M Nguyen, Mark~S Squillante, Ivan Oseledets,
  and Luca Daniel.
\newblock Proven: Certifying robustness of neural networks with a probabilistic
  approach.
\newblock {\em arXiv preprint arXiv:1812.08329}, 2018.

\bibitem{wu2004learning}
Jianxin Wu, James~M Rehg, and Matthew~D Mullin.
\newblock Learning a rare event detection cascade by direct feature selection.
\newblock In {\em Advances in Neural Information Processing Systems}, pages
  1523--1530, 2004.

\bibitem{yang2004neural}
Simon~X Yang and Chaomin Luo.
\newblock A neural network approach to complete coverage path planning.
\newblock {\em IEEE Transactions on Systems, Man, and Cybernetics, Part B
  (Cybernetics)}, 34(1):718--724, 2004.

\bibitem{zhao2016accelerated}
Ding Zhao, Henry Lam, Huei Peng, Shan Bao, David~J LeBlanc, Kazutoshi Nobukawa,
  and Christopher~S Pan.
\newblock Accelerated evaluation of automated vehicles safety in lane-change
  scenarios based on importance sampling techniques.
\newblock {\em IEEE transactions on intelligent transportation systems},
  18(3):595--607, 2016.

\end{thebibliography}

\end{document}